\documentclass[journal]{IEEEtran}
\ifCLASSINFOpdf
\else
\fi

\usepackage{amsmath, graphicx, bm}
\usepackage{cases}
\usepackage{float, subfigure}
\usepackage{color}
\usepackage{epstopdf}
\usepackage{mathrsfs}
\usepackage{hyperref}
\usepackage{amsthm}
\usepackage{amsfonts}
\usepackage{cite, url}
\usepackage[all,cmtip]{xy}
\usepackage{color}
\usepackage[noend]{algpseudocode}
\usepackage{algorithmicx,algorithm}
\usepackage{tabularx}
\usepackage{multirow}
\usepackage[normalem]{ulem} % to use \sout
\usepackage[switch]{lineno}

\newtheorem{theorem}{Theorem}

\newtheorem{lemma}{Lemma}
\newtheorem{proper}{Property}

\newtheorem{assumption}{Assumption}

\newcommand{\R}{{\mathbb{R}}}
\newcommand{\E}{\mathbb{E}}

\begin{document}
%
% paper title
% Titles are generally capitalized except for words such as a, an, and, as,
% at, but, by, for, in, nor, of, on, or, the, to and up, which are usually
% not capitalized unless they are the first or last word of the title.
% Linebreaks \\ can be used within to get better formatting as desired.
% Do not put math or special symbols in the title.
\title{Byzantine-Robust Variance-Reduced Federated Learning over Distributed Non-i.i.d. Data}
%
%
% author names and IEEE memberships
% note positions of commas and nonbreaking spaces ( ~ ) LaTeX will not break
% a structure at a ~ so this keeps an author's name from being broken across
% two lines.
% use \thanks{} to gain access to the first footnote area
% a separate \thanks must be used for each paragraph as LaTeX2e's \thanks
% was not built to handle multiple paragraphs
%

	\author{Jie~Peng,
	Zhaoxian~Wu,
	Qing~Ling,
	and~Tianyi~Chen
	% <-this % stops a space
	\thanks{Jie Peng, Zhaoxian Wu and Qing Ling are with the School of Computer Science and Engineering, the Guangdong Provincial Key Laboratory of Computational Science, Sun Yat-Sen University, Guangzhou 510006, China, and also with the Pazhou Lab, Guangzhou 510335, China. The work of Qing Ling is supported in part by National Natural Science Foundation of China under grant 61973324, Guangdong Basic and Applied Basic Research Foundation under grant 2021B1515020094, and Guangdong Provincial Key Laboratory of Computational Science, Sun Yat-Sen University under grant 2020B1212060032. (e-mail:
		pengj95@mail2.sysu.edu.cn; wuzhx23@mail2.sysu.edu.cn; lingqing556@mail.sysu.edu.cn).}% <-this % stops a space
	\thanks{Tianyi Chen is with the Department of Electrical, Computer, and Systems
		Engineering, Rensselaer Polytechnic Institute, Troy, NY 12180 USA (e-mail:
		chent18@rpi.edu).}}

% note the % following the last \IEEEmembership and also \thanks -
% these prevent an unwanted space from occurring between the last author name
% and the end of the author line. i.e., if you had this:
%
% \author{....lastname \thanks{...} \thanks{...} }
%                     ^------------^------------^----Do not want these spaces!
%
% a space would be appended to the last name and could cause every name on that
% line to be shifted left slightly. This is one of those "LaTeX things". For
% instance, "\textbf{A} \textbf{B}" will typeset as "A B" not "AB". To get
% "AB" then you have to do: "\textbf{A}\textbf{B}"
% \thanks is no different in this regard, so shield the last } of each \thanks
% that ends a line with a % and do not let a space in before the next \thanks.
% Spaces after \IEEEmembership other than the last one are OK (and needed) as
% you are supposed to have spaces between the names. For what it is worth,
% this is a minor point as most people would not even notice if the said evil
% space somehow managed to creep in.

% The paper headers
\markboth{IEEE Transactions on Cybernetics}%
{Peng \MakeLowercase{\textit{et al.}}: Byzantine-Robust Variance-Reduced Federated Learning over Distributed Non-i.i.d. Data}
% The only time the second header will appear is for the odd numbered pages
% after the title page when using the twoside option.
%
% *** Note that you probably will NOT want to include the author's ***
% *** name in the headers of peer review papers.                   ***
% You can use \ifCLASSOPTIONpeerreview for conditional compilation here if
% you desire.

% If you want to put a publisher's ID mark on the page you can do it like
% this:
%\IEEEpubid{0000--0000/00\$00.00~\copyright~2015 IEEE}
% Remember, if you use this you must call \IEEEpubidadjcol in the second
% column for its text to clear the IEEEpubid mark.

% use for special paper notices
%\IEEEspecialpapernotice{(Invited Paper)}

% make the title area
\maketitle

% As a general rule, do not put math, special symbols or citations
% in the abstract or keywords.
\begin{abstract}
	We consider the federated learning problem where data on workers are not independent and identically distributed (i.i.d.). During the learning process, an unknown number of Byzantine workers may send malicious messages to the central node, leading to remarkable learning error. Most of the Byzantine-robust methods address this issue by using robust aggregation rules to aggregate the received messages, but rely on the assumption that all the regular workers have i.i.d. data, which is not the case in many federated learning applications. In light of the significance of reducing stochastic gradient noise for mitigating the effect of Byzantine attacks, we use a resampling strategy to reduce the impact of both inner variation (that describes the sample heterogeneity on every regular worker) and outer variation (that describes the sample heterogeneity among the regular workers), along with a stochastic average gradient algorithm to gradually eliminate the inner variation. The variance-reduced messages are then aggregated with a robust geometric median operator. We prove that the proposed method reaches a neighborhood of the optimal solution at a linear convergence rate and the learning error is determined by the number of Byzantine workers. Numerical experiments corroborate the theoretical results and show that the proposed method outperforms the state-of-the-arts in the non-i.i.d. setting.
\end{abstract}

% Note that keywords are not normally used for peerreview papers.

% For peer review papers, you can put extra information on the cover
% page as needed:
% \ifCLASSOPTIONpeerreview
% \begin{center} \bfseries EDICS Category: 3-BBND \end{center}
% \fi
%
% For peerreview papers, this IEEEtran command inserts a page break and
% creates the second title. It will be ignored for other modes.
\IEEEpeerreviewmaketitle

\section{Introduction}

\IEEEPARstart{W}{ith} the rapid increase of data volume and computing power, the past decades have witnessed the explosive development of machine learning. However, many machine learning methods require a central server or a cloud to collect data from various owners and train models in a centralized manner, leading to serious privacy concerns. To address this issue, \textit{federated learning} keeps data private at their owners and carries out machine learning tasks locally \cite{konevcny2016federatedoptimization, mcmahan2016federated,yang2019federated,kairouz2019advances, park2019wireless, zeng2020federated}. Every data owner (called worker thereafter) performs local computation based on its local data and sends the results (such as local models, gradients, stochastic gradients, etc) to the central server. The central server (called central node thereafter) aggregates the results received from the workers and updates the global model.
%To solve the distributed computation of algebraic Riccati inequalities (ARIs), [Distributed Optimization Design] designs a distributed optimization method to compute ARI, tackling the coupled information structure of ARI.

Nevertheless, the distributed nature of federated learning makes it vulnerable to attacks \cite{kairouz2019advances}. Due to the heterogeneity of federated learning systems, workers are not all reliable. Some workers might be malfunctioning or even adversarial, and send malicious messages to the central node. This paper considers the classical Byzantine attack model \cite{LamportSP82}, where an unknown number of Byzantine workers are omniscient, collude with each other, and send arbitrary malicious messages. Moreover, the identities of Byzantine workers are unknown to the central node. Misled by the Byzantine workers, the aggregation at the central node is problematic, such that the federated learning method may converge to an unsatisfactory model or even diverge. For instance, the popular distributed stochastic gradient descent (SGD) method fails at presence of Byzantine attacks from a single worker \cite{geomed}.

Most of the Byzantine-robust distributed methods modify distributed SGD to handle Byzantine attacks, by replacing mean aggregation with robust aggregation at the central node \cite{yang2020adversary,xie2020fall,cao2019distributed}. When the data on the workers are independent and identically distributed (i.i.d.) and the local cost functions are in the same form, stochastic gradients computed at the same point are i.i.d. too. Therefore, various robust aggregation rules, such as geometric median \cite{geomed}, can be applied to alleviate the effect of statistically biased messages sent by Byzantine workers. Unfortunately, the data on the workers are often non-i.i.d. in federated learning applications \cite{zhao2018federated,li2020federated}. Thus, the messages sent by the regular workers are no longer i.i.d. such that handling Byzantine attacks becomes more challenging.

\subsection{Our contributions}

In the Byzantine-robust federated learning setting, this paper considers the distributed finite-sum minimization problem when the workers have non-i.i.d. data. Our contributions are summarized as follows.

\noindent \textbf{C1)} In light of the significance of reducing stochastic gradient noise for mitigating the effect of Byzantine attacks, we propose a Byzantine-robust variance-reduced SGD method, using a resampling strategy to reduce the impact of both inner variation (that describes the sample heterogeneity on every regular worker) and outer variation (that describes the sample heterogeneity among the regular workers), while a stochastic average gradient algorithm (SAGA) to fully eliminate the inner variation. The variance-reduced messages are then aggregated with a robust geometric median operator.

\noindent \textbf{C2}) We prove that the proposed method reaches a neighborhood of the optimal solution at a linear convergence rate, and the learning error is explicitly determined by the number of Byzantine workers.

%Numerical experiments corroborate the theoretical results and show that the learning error of the proposed method is much smaller than those given by the state-of-the-art methods in the non-i.i.d. setting.

%and the learning error is much smaller than those given by the state-of-the-art methods in the non-i.i.d. setting.

\noindent \textbf{C3}) We conduct numerical experiments on convex and nonconvex federated learning problems and in i.i.d. and non-i.i.d. settings. The experimental results corroborate the theoretical findings and show that the proposed method outperforms the state-of-the-arts, especially in the non-i.i.d. setting.

\subsection{Related works}

Federated learning, in a nutshell, belongs to the category of distributed optimization, for which there are extensive works, ranging from machine learning, signal processing to system control. For example, \cite{zhu2018projected} proposes a projected primal-dual method to solve the distributed constrained nonsmooth convex optimization problem. Aiming at the distributed time-varying formation problem, \cite{sun2021time} proposes a distributed framework with the help of time-varying optimization methods. Below we focus on the related works in Byzantine-robustness of distributed algorithms, especially those handling non-i.i.d. data.

Byzantine-robust distributed machine learning has attracted much attention in recent years. In distributed SGD, the central node aggregates messages received from the workers by taking average and uses the mean as aggregated gradient direction. The mean aggregation, however, is vulnerable to Byzantine attacks. Existing Byzantine-robust distributed methods mostly extend the distributed SGD with robust aggregation rules, such as geometric median \cite{geomed}, normalized aggregation \cite{jin2019distributed}, coordinate-wise median \cite{cmctm}, coordinate-wise trimmed mean \cite{cmctm}, Krum \cite{krumm-krum}, multi-Krum \cite{krumm-krum}, Bulyan \cite{bulyan}, to name a few. Algorithms leveraging second-order information has also be considered in \cite{cao2020distributed}.
Another approach is to detect and discard outliers from the received messages \cite{rodriguez2020dynamic, azulay2020holdout, li2020learning}. When the received stochastic gradients from the regular workers satisfy the i.i.d. assumption, the statistically different malicious messages from the Byzantine workers can be detected and discarded, or their negative effect can be alleviated by robust aggregation rules. For network anomaly detection in distributed online optimization, \cite{miao2018distributed} proposes a distributed one-class support vector machine method.

%and \cite{el2020genuinely}
Aiming at Byzantine-robustness with distributed non-i.i.d. data, \cite{RSA} proposes a robust stochastic aggregation method that employs model aggregation rather than stochastic gradient aggregation. Forced by the introduced consensus constraints, the regular workers and the central node shall reach consensus on their local models, no matter the local data are i.i.d. or not. \cite{dong2020communication} also imposes asymptotic consensus between the regular workers and the central node, and proposes a Byzantine-robust proximal stochastic gradient method. \cite{ghosh2019robust} divides the cost functions of the workers into several clusters, such that within each cluster the i.i.d. assumption approximately holds. Then robust aggregation can be applied within each cluster. \cite{he2020byzantine} introduces a resampling strategy to reduce the heterogeneity of the received messages in the non-i.i.d. setting.
%\cite{chen2020distributed} finds that the coordinate-wise median aggregator with some random noise is helpful in solving distributed heterogeneous data.
%[Comment by QL: This paper does consider Byzantine attacks in theory and simulation.]
%\cite{data2020byzantine} gives a deterministic condition on outer variations and provides concrete bounds of it.
%[Comment by QL: The proposed method does not really handles non-i.i.d. data.]

In both i.i.d. and non-i.i.d. settings, variance of the messages sent by the regular workers plays a critical role to Byzantine-robustness. Larger variance means that the malicious messages are harder to distinguish. Theoretically, the variance can be classified into inner variation that describes the sample heterogeneity on every regular worker and outer variation that describes the sample heterogeneity among the regular workers. In the i.i.d. setting the inner variation often dominates, while in the non-i.i.d. setting the outer variation can be large. \cite{byrdsaga} and its extended version \cite{byrdsaga-arxiv} use SAGA to correct the stochastic gradients and aggregate the corrected stochastic gradients with geometric median. It has been proven that the impact of inner variation is fully eliminated from the learning error. \cite{khanduri2019byzantine} combines stochastic variance-reduced gradient (SVRG) with robust aggregation to solve distributed non-convex problems. \cite{el2020distributed} proves that the momentum method can reduce the variance of the stochastic gradients for the regular workers relative to their norm, and is thus helpful to Byzantine-robustness. \cite{karimireddy2020learning} also claims that applying the momentum method can reduce the variance and enhance Byzantine-robustness. The resampling strategy used in \cite{he2020byzantine} is able to reduce the impact of inner and outer variations simultaneously.

The work of this paper is closely related to \cite{byrdsaga-arxiv} and \cite{he2020byzantine}, but not a simple combination of them. First, the theme of \cite{byrdsaga-arxiv} is to investigate the importance of variance reduction to Byzantine-robustness for i.i.d. data, while this work considers non-i.i.d. data. \cite{he2020byzantine} studies reducing the impact of inner and outer variations by resampling, but is unable to eliminate the impact of inner variation. Second, the theoretical analysis is more difficult than that in \cite{byrdsaga-arxiv} and \cite{he2020byzantine} because of the entanglement of variance reduction and resampling. We will highlight the differences of our lemmas and theorems in Section \ref{sec:theory}.

\section{Problem Formulation and Preliminaries}
Consider a federated learning system consisting of one central node and $W:= R+B$ workers, among which $R$ workers are regular and $B$ workers are Byzantine. However, the numbers and identities of Byzantine workers are unknown to the central node. During the learning process, the Byzantine workers can collude to send arbitrary malicious messages to the central node \cite{LamportSP82}. Denote the sets of regular and Byzantine workers as $\mathcal{R}$ and $\mathcal{B}$, respectively, with $R = |\mathcal{R}|$ and $B = |\mathcal{B}|$. Denote $\mathcal{W}:= \mathcal{R} \cup \mathcal{B} = \{1, \cdots, W\}$ as the set of all workers. The goal is to find an optimal solution to the finite-sum minimization problem
\begin{align}\label{eq1}
	x^* =  \arg\min_{x} f(x) :=\frac{1}{R} \sum_{w \in \mathcal{R}} f_w(x),
\end{align}
with
\begin{align}\label{eq2}
	f_w(x) := \frac{1}{J} \sum_{j=1}^J f_{w,j}(x).
\end{align}
Here $x  \in \mathbb{R}^p$ is the optimization variable, and $f_w(x)$ is the local cost function of regular worker $w$ averaging the costs $f_{w,j}(x)$ of $J$ samples. The samples are not necessarily i.i.d. across the regular workers. This finite-sum minimization form arises in many federated learning applications \cite{yang2019federated}.

% QL: cite federated learning papers - done

When the Byzantine workers are absent, distributed SGD \cite{sgd} is a popular method to solve \eqref{eq1}. The update of distributed SGD is given by
\begin{align}\label{sgd}
	x^{k+1} = x^{k} - \gamma \cdot \frac{1}{W}\sum_{w=1}^W f'_{w, i_w^k} (x^k),
\end{align}
where $\gamma$ is the step size and $i_w^k$ is the sample index chosen by worker $w$ at time $k$. Upon receiving $x^k$ from the central node, every worker $w$ computes a stochastic gradient $f'_{w, i_w^k} (x^k)$ and returns it to the central node. The central node then averages the received stochastic gradients and updates $x^{k+1}$. However, distributed SGD is vulnerable to Byzantine attacks. Even there is only one Byzantine worker, it can replace the true stochastic gradient with a malicious message, such that the average at the central node becomes zero or infinitely large \cite{yang2020adversary}.

\noindent\textbf{Geometric median.} 	To address this issue, most of the Byzantine-robust methods aggregate the received messages with robust aggregation rules. One popular approach is geometric median \cite{geomed}. To be specific, define $g_w^k$ as the message sent by worker $w$ at time $k$ to the central node, given by
\begin{align}\label{ggg}
	g_w^k =
	\left\{
	\begin{aligned}[]
		&f'_{w, i_w^k} (x^k),  && w \in \mathcal{R}, \\
		&*, && w \in \mathcal{B}.
	\end{aligned}
	\right.
\end{align}
When $w$ is a regular worker, it sends the true stochastic gradient $f'_{w, i_w^k} (x^k)$. Otherwise when $w$ is a Byzantine worker, it sends an arbitrary message denoted by $*$. Upon receiving all the messages, the central node calculates the geometric median of $\{g_w^k, w \in \mathcal{W}\}$, which is
\begin{align}\label{eq:geomed}
	\mathrm{geomed}\left(\{g_w^k, w \in \mathcal{W}\}\right) := \arg\min_{g}\sum_{w=1}^{W} \|g - g_w^k\|_2.
\end{align}
Then, the central node updates $x^{k+1}$ by replacing the average in \eqref{sgd} with the geometric median, as
\begin{align}\label{eq:geomed_update}
	x^{k+1} = x^{k} - \gamma \cdot \mathrm{geomed}\left(\{g_w^k, w \in \mathcal{W}\}\right).
\end{align}
It has been proved that when the regular workers have i.i.d. data, incorporating distributed SGD and geometric median allows us to tolerate Byzantine attacks when less than half of the workers are malicious, namely, $B < \frac{W}{2}$ \cite{geomed}. However, when the regular workers have non-i.i.d. data, simply combining geometric median with distributed SGD is unable to defend against Byzantine attacks \cite{RSA}.

\noindent \textbf{Why large variance hurts?} To understand the importance of variance reduction to robust aggregation, we review the following concentration property of geometric median from \cite[Lemma 1]{byrdsaga-arxiv}. %we show the concentration property of geometric median in the following lemma.

\begin{proper} \label{lemma:geometric_error}
	Let $\{z_w, w\in\mathcal{W}\}$ be a set of random vectors. It holds when $B < \frac{W}{2}$ that
	\begin{align} \label{inequality:geovoting-error-exp}
		&\E \|\mathrm{geomed}\left(\{z_w, w\in\mathcal{W}\} \right) - \bar{z}\|^2    \\
		\le & \frac{C_\alpha^2}{R} \sum_{w\in\mathcal{R}}{ \E\|z_w-\E z_w\|^2 }
		+ \frac{C_\alpha^2}{R} \sum_{w\in\mathcal{R}}{ \|\E z_w- \bar{z}\|^2}, \nonumber
	\end{align}
	where $\bar{z}:=\frac {1} {R}\sum_{w\in\mathcal{R}} \E z_w$, $\alpha :=\frac{B}{W}$, $C_\alpha :=\frac{2-2\alpha}{1-2\alpha}$, and $\E$ is taken over the random vectors.
\end{proper}

Suppose $z_w = g_w^k$ as defined \eqref{ggg}, such that $z_w$ is the true stochastic gradient when $w$ is regular, and an arbitrary malicious message when $w$ is Byzantine. The left-hand side of \eqref{inequality:geovoting-error-exp} denotes the mean-square error of the geometric median relative to the average of the true stochastic gradients. The right-hand side of \eqref{inequality:geovoting-error-exp} contains two terms. The first term now refers to the cross-sample variance (a.k.a. inner variation), while the second term refers to the cross-worker variance (a.k.a. outer variation). When either the inner variation or the outer variation is large, the geometric median aggregation yields a poor direction that eventually leads to large learning error, as illustrated in Fig. \ref{Fig:variation}. This fact motivates Byrd-SAGA \cite{byrdsaga,byrdsaga-arxiv}. It has been shown in \cite{byrdsaga,byrdsaga-arxiv} that SAGA effectively eliminates the inner variation, i.e., the first term at the right-hand side of \eqref{inequality:geovoting-error-exp}. However, for the non-i.i.d. case, the second term at the right-hand side of \eqref{inequality:geovoting-error-exp} can be large and still deteriorate the performance of robust aggregation.

\begin{figure}
	%\centering
	\includegraphics[scale=0.45]{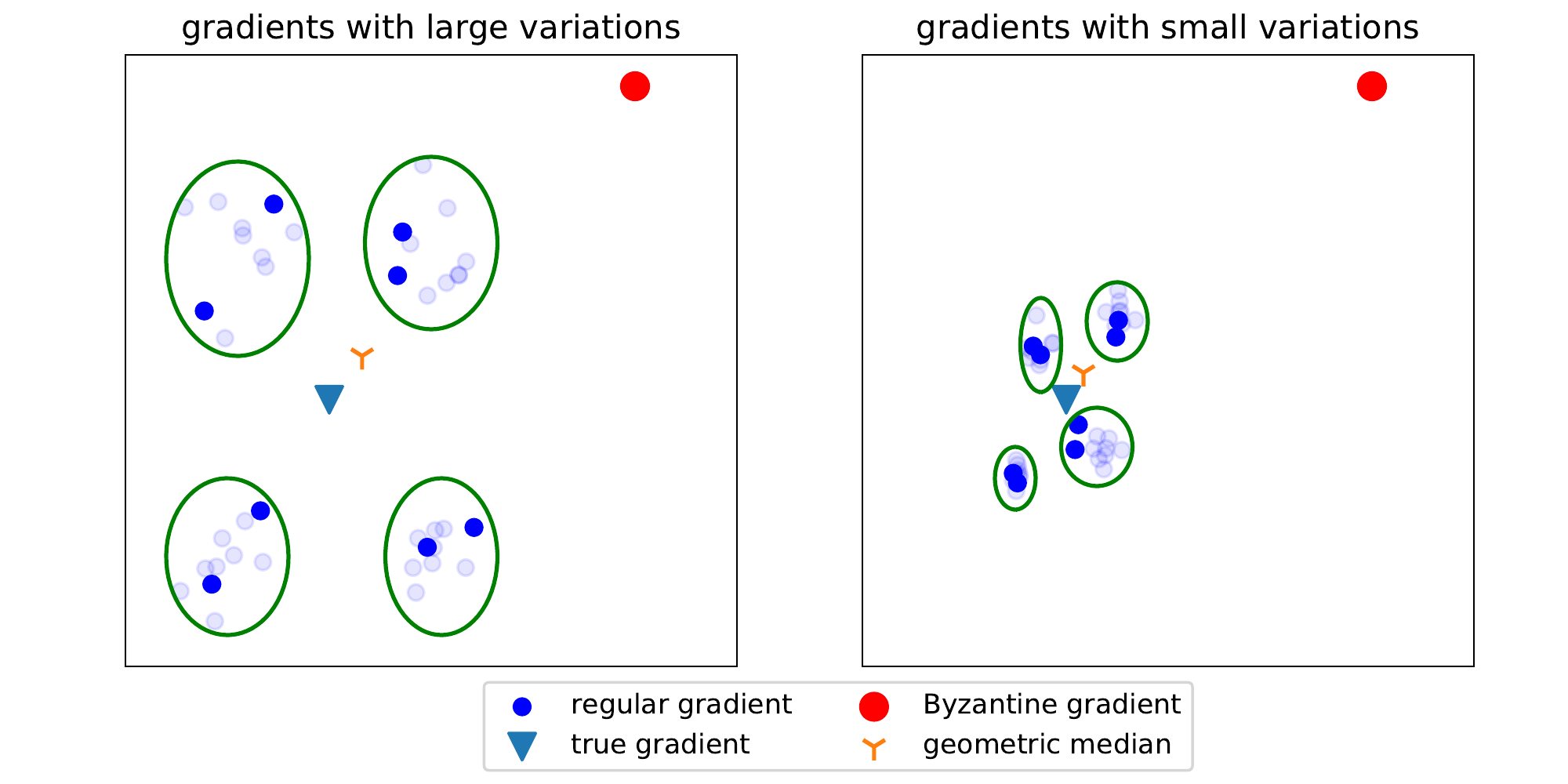}
	\caption{An illustration of the impact of inner and outer variations on geometric median under Byzantine attacks. Suppose 8 regular workers sample stochastic gradients from 4 distributions. When the variations are large, the Byzantine-free true average is far away from the geometric median under Byzantine attacks. In contrast, when the variations are small, the gap is small too.}
	\label{Fig:variation}
\end{figure}

\section{Algorithm Development}

In light of the fact that the messages with large variance affect the robust aggregation performance, we propose to \emph{simultaneously reduce the inner and outer variations} so as to achieve better performance in the non-i.i.d. setting. Specifically, we use the variance-reduced stochastic gradients at the worker side to gradually eliminate the inner variation, and use a resample-then-aggregate strategy at the central node to reduce outer variation.

Our proposed method is described as follows. At time $k$, the central node broadcasts its current variable $x^k$ to all the workers. Instead of sending the stochastic gradients to the central node, the regular workers send the corrected, variance-reduced stochastic gradients. To be specific, every regular worker stores the most recent stochastic gradient for every of its samples. When regular worker $w$ selects the sample index $i_w^k$ at time $k$, the corrected stochastic gradient is
\begin{align} \label{eq:saga}
	f'_{w, i_w^k}(x^k) - f'_{w, i_w^k}(\phi_{w, i_w^k}^k) + \frac{1}{J} \sum_{j=1}^J  f'_{w, j}(\phi_{w, j}^k),
\end{align}
where
\begin{align}
	\label{rule:phi}
	\phi_{w, j}^{k+1} =
	\left\{
	\begin{aligned}
		&\phi_{w, j}^{k}, && j \neq i_w^k, \\
		&x^k, && j = i_w^k.
	\end{aligned}
	\right.
\end{align}
That is, the stochastic gradient $f'_{w, i_w^k}(x^k)$ is corrected by first subtracting the previously stored stochastic gradient $f'_{w, i_w^k}(\phi_{w, i_w^k}^k)$ for sample $i_w^k$, and then adding the average of all the stored stochastic gradients $f'_{w, j}(\phi_{w, j}^k)$. At time $k$, denote $v_w^k$ as the corrected stochastic gradient if $w$ is regular and an arbitrary message $*$ if $w$ is Byzantine, given by
\begin{align}\label{vvv}
	\!\!\!\!	v_w^k =
	\left\{
	\begin{aligned}
		&f'_{w, i_w^k}(x^k) - f'_{w, i_w^k}(\phi_{w, i_w^k}^k) && \\
		&\hspace{3em}+ \frac{1}{J} \sum_{j=1}^J  f'_{w, j}(\phi_{w, j}^k), && w \in \mathcal{R}, \\
		&*, && w \in \mathcal{B}.
	\end{aligned}
	\right.
\end{align}

\begin{algorithm}[t]
	\small
	\caption{Resampling $\mathbf{RS}\left(\{v_w^k, w \in \mathcal{W}\}, s\right)$ }
	\label{algorithm:resample}
	{\textbf{Input:}} variance-reduced worker gradients $\{v_w^k, w \in \mathcal{W}\}$, $s$
	\begin{algorithmic}[1]
		\State Initialize $\{c_w=0, w \in \mathcal{W}\}$
		\For{$w = 1, \cdots$, $W$}
		\For{$\zeta=1, \cdots, s$}
		%           \While{\textbf{true}}
		\State Select $w_\zeta \sim \textbf{Uniform}( \{w' | w' \in \mathcal{W}, c_{w'} < s \} )$
		\State $c_{w_\zeta} = c_{w_\zeta} +1$
		\EndFor
		\State Compute average $\tilde{v}_w^k = \frac{1}{s} \sum_{\zeta=1}^s v_{w_\zeta}^k$
		\EndFor
		\State Return $\{\tilde{v}_w^k, w \in \mathcal{W}\}$
	\end{algorithmic}
\end{algorithm}

Building upon \cite{he2020byzantine}, we leverage a resample-then-aggregate strategy to reduce the outer variation in the non-i.i.d. case. After receiving the messages $\{v_w^k, w \in \mathcal{W}\}$ from all the workers, the central node takes $W$ rounds of sampling, and uniformly at random samples $s$ messages at every round. The $W$ rounds of sampling follows sampling with $s$-replacement mechanism: when one message has been sampled for $s$ times, it will not be sampled again. At the end of each round $w$, the central node averages the $s$ sampled messages to calculate the new message $\tilde{v}_{w}^k$. See a summary in Algorithm 1.

We denote this resampling procedure with $s$-replacement as
\begin{align}\label{resampling}
	\{\tilde{v}_w^k, w \in \mathcal{W}\} = \mathbf{RS}\left(\{v_w^k, w \in \mathcal{W}\}, s\right).
\end{align}
Compared to the original messages $\{v_w^k, w \in \mathcal{W}\}$, the new messages $\{\tilde{v}_w^k, w \in \mathcal{W}\}$ have smaller variance due to the averaging step. In addition, at most $sB$ new messages are contaminated by the malicious messages $\{v_w^k, w \in \mathcal{B}\}$, and the variance of uncontaminated new messages is also smaller than that of $\{v_w^k, w \in \mathcal{R}\}$. Fig. \ref{Fig:resampling} gives a deterministic example, showing that resampling indeed reduces the outer variation. For the stochastic case, the inner variation is reduced too.

\begin{figure}
	\centering
	\includegraphics[scale=0.35]{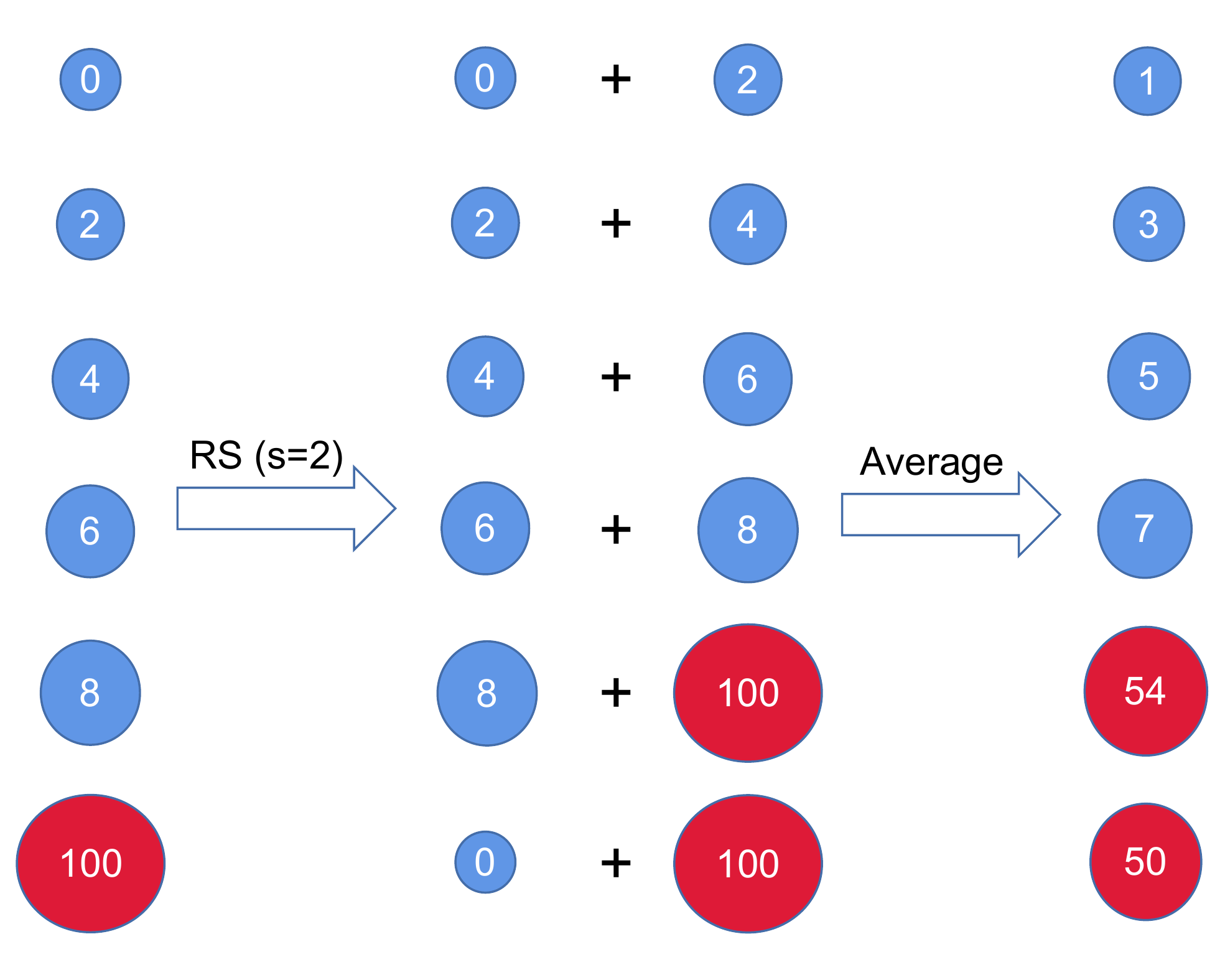}
	\caption{An illustration the variance-reducing effect of resampling, where 5 true messages (marked in blue) and 1 malicious message (marked in red) enters the resampling procedure with $s=2$. After the averaging step, at most 2 messages are contaminated. The variation of all the 6 original messages is $3860/3$, and that of all the 6 messages after resampling is $1550/3$. In addition, the variation of all the 5 original true messages is 8, and that of all the 4 uncontaminated messages after resampling is 5.}
	\label{Fig:resampling}
\end{figure}

Finally, the central node updates the variable from
\begin{align}\label{aggregation}
	x^{k+1} = x^k - \gamma \cdot \mathrm{geomed}\left(\{\tilde{v}_w^k, w \in \mathcal{W}\}\right).
\end{align}
See a summary of our proposed method in Algorithm 2.

\begin{algorithm}[t]
	\small
	\caption{Byzantine-Robust Variance-Reduced Federated Learning over Distributed Non-i.i.d. Data}
	\label{algorithm:RS-Byrd-SAGA}
	\centering \textbf{central node}:
	\begin{algorithmic}[1]
		\State Input $x^0 \in \mathbb{R}^p$, $\gamma$, $s$. At time $k$:
		\State Broadcast its current variable $x^k$ to all workers
		\State Receive messages $\{v_w^k, w \in \mathcal{W}\}$ from all workers
		\State Generate $\{\tilde{v}_w^k, w \in \mathcal{W}\}$ from $\{v_w^k, w \in \mathcal{W}\}$ according to Algorithm \ref{algorithm:resample} with parameter $s$
		\State Update $x^{k+1} = x^k - \gamma \cdot \mathrm{geomed}\left(\{\tilde{v}_w^k, w \in \mathcal{W}\}\right)$
	\end{algorithmic}
	\vspace*{1\baselineskip}
	\centering\textbf{Regular worker $w$}:
	\begin{algorithmic}[1]
		\State Initialize $\{f'_{w, j}(\phi_{w, j}^0) = f'_{w,j}(x^0), ~ j \in \{1, \cdots, J\} \}$. At time $k$:
		\State Receive variable $x^k$ from central node
		\State Compute $\bar{g}_w^k = \frac{1}{J} \sum_{j=1}^J f'_{w, j}(\phi_{w, j}^k)$.
		\State Sample $i_w^k$ uniformly randomly from $j \in \{1, \cdots, J\}$
		\State Update $v_w^k = f'_{w, i_w^k}(x^k) -  f'_{w, i_w^k}(\phi_{w, i_w^k}^k) + \bar{g}_w^k$
		\State Store stochastic gradient $f'_{w, i_w^k}(\phi_{w, i_w^k}^k) = f'_{w, i_w^k}(x^k)$
		\State Send corrected stochastic gradient $v_w^k$ to central node
	\end{algorithmic}
\end{algorithm}

\section{Theoretical Analysis}
\label{sec:theory}
In this section, we theoretically analyze the performance of the proposed method, as well as investigate the effect of reducing the inner and outer variations on robust aggregation.

\subsection{Assumptions}

We start from introducing three assumptions on the sample costs $\{f_{w,j}, w \in \mathcal{R}, j = 1, \cdots, J\}$.
\begin{assumption}\label{assumption1}
	(Strong convexity and Lipschitz continuous gradients) The function $f$ is $\mu$-strong convex and has $L$-Lipschitz continuous gradients. Namely, for any $x,y \in \R^p$, it holds that
	\begin{align}
		f(y) \geq f(x) + \langle f'(x), y - x\rangle + \frac{\mu}{2} \|y-x\|^2,
	\end{align}
	and
	\begin{align}
		\|f'(y) - f'(x)\| \leq L \| y - x\|.
	\end{align}
\end{assumption}

\begin{assumption}\label{assumption2}
	(Bounded outer variation) For any $x \in \R^p$, the variation of the local gradients at the regular workers with respect to the global gradient is upper-bounded by
	\begin{align}
		\frac{1}{R} \sum_{w \in \mathcal{R}}\| f'_w(x) - f'(x) \|^2 \leq \delta^2.
	\end{align}
\end{assumption}

\begin{assumption}\label{assumption3}
	(Bounded inner variation) For every regular worker $w \in \mathcal{R}$ and any $x \in \R^p$, the variation of its stochastic gradients with respect to its local gradient is upper-bounded by
	\begin{align}
		\E \|f'_{w, i_w^k}(x) - f'_w(x)\|^2 \leq \sigma^2, ~ \forall w \in \mathcal{R}.
	\end{align}
	Here we use $\E$ to denote the expectation with respect to all random variables $i_w^k$.
\end{assumption}

Assumption \ref{assumption1} is standard for analyzing distributed learning algorithms. Assumptions \ref{assumption2} and \ref{assumption3} bound the inner variation that describes the sample heterogeneity on every regular worker and the outer variation that describes the sample heterogeneity among the regular workers, respectively. When all the samples at the regular workers are identical, both the inner variation and the outer variation are zero. When the distributed data are i.i.d. the inner variation is often larger than the outer variation, while for the non-i.i.d. case the outer variation usually dominates.

% QL: cite Ji Liu's paper

\subsection{Concentration properties}
\label{sec:concentration}
Interestingly, augmented with the resampling strategy, geometric median shows better dependency on both the inner variation and the outer variation than that without resampling (as shown in Property 1). We give the following lemma.

\begin{lemma}
	\label{lemma:resampling-geomed}
	Let $\{z_w, w\in \mathcal{W}\}$ be a subset of random vectors distributed in a normed vector space and the random vectors in $\{z_w, w\in \mathcal{R}\}$ are independent. Generate from $\{z_w, w\in \mathcal{W}\}$ a new set $\{\tilde{z}_w, w\in \mathcal{W}\}$ using the resampling strategy with $s$-replacement. It holds when $B < \frac{W}{2s}$ that
	\begin{align}\label{lm1-1}
		&\E \| \mathrm{geomed}\left(\{\tilde{z}_w, w\in\mathcal{W}\} \right) - \bar{z}\|^2   \\
		\leq & \left(d+\frac{1-d}{R}\right)\frac{C_{s\alpha}^2}{R} \sum_{w\in\mathcal{R}}{ \E\|z_w-\E z_w\|^2 } \nonumber \\
		&+ \frac{dC_{s\alpha}^2}{R} \sum_{w\in\mathcal{R}}{ \|\E z_w- \bar{z}\|^2},  \nonumber
	\end{align}
	where $\bar{z}:=\frac {1} {R}\sum_{w\in\mathcal{R}} \E z_w$, $\alpha :=\frac{B}{W}$, $C_{s\alpha}:= \frac{2-2s\alpha}{1-2s\alpha}$, $d: = \frac{W-1}{s W - 1}$, and $\E$ is taken over the random vectors and the resampling process.
\end{lemma}

Comparing the bounds of mean-square errors given by \eqref{inequality:geovoting-error-exp} and \eqref{lm1-1}, we observe that the coefficients change. When $s=1$, the two bounds are identical. When $s > 1$, the coefficients in \eqref{lm1-1} smaller than those in \eqref{inequality:geovoting-error-exp} if $\alpha$ is sufficiently small, meaning that the inner variation and the outer variation are simultaneously reduced. The cost, however, is the capability of tolerating a smaller fraction of Byzantine workers. We will numerically depict the effect of variance reduction in Fig. \ref{Fig:AsymptoticLearningError} when discussing the main theorem.

% QL: cost of non-i.i.d. appears in $s$ - done

Again, suppose $z_w = g_w^k$ as defined in \eqref{ggg}. Then, the right-hand side of \eqref{lm1-1} contains two terms, one for inner variation and another for outer variation.
Resampling can reduce the dependency on both the inner and outer variations, but cannot eliminate the inner variation.
We shall see in the main theorem that, the unique coupling of SAGA and resampling in our proposed method is able to eliminate the inner variation, and hence leads to reduced learning error.

\noindent \textbf{Remark 1}.
Above, we have discussed the connection between Lemma \ref{lemma:resampling-geomed} and \cite[Lemma 1]{byrdsaga-arxiv}. Technically, to prove Lemma \ref{lemma:resampling-geomed}, we need to apply \cite[Proposition 1]{he2020byzantine} to guarantee the reduction of upper bound for the mean-square error as in \eqref{lm1-1}. We also utilize the fact that the random vectors are independent to obtain a tight bound.

%but is with the factor $d$ at the right-hand side of \eqref{lm1-1}. It means that with resampling,  mean-square errors

The intermediate result of Lemma \ref{lemma:resampling-geomed} is the following lemma, which characterizes how the averaged mean-square error of a number of vectors depends on the inner and outer variations after resampling.

\begin{lemma}
	\label{lemma:resampling-inner-outer-variation}
	Let $\{z_w, w\in \mathcal{W}\}$ be a subset of random vectors distributed in a normed vector space and the random vectors in $\{z_w, w\in \mathcal{R}\}$ are independent. Generate from $\{z_w, w\in \mathcal{W}\}$ a new set $\{\tilde{z}_w, w\in \mathcal{W}\}$ using the resampling strategy with $s$-replacement. When $B < \frac{W}{s}$, there exists a set $\mathcal{R}' \subseteq \mathcal{W}$ with at least $W-sB$ elements, such that for any $w' \in \mathcal{R}'$, it holds that
	\begin{align} \label{lemma:resampling-inner-outer-variation-inequatlity}
		&\frac{1}{|\mathcal{R'}|}\sum_{w'\in\mathcal{R}'} \E \|\tilde{z}_{w'}-\bar z\|^2  \\
		= & \left(d+\frac{1-d}{R}\right)\frac{1}{R}\sum_{w\in\mathcal{R}}
		\E \left\|z_w - \E z_w\right\|^2 \nonumber\\
		&+ \frac{d}{R}\sum_{w\in\mathcal{R}} \left\|\E z_w - \bar z\right\|^2,\nonumber
	\end{align}
	where $\bar{z}:=\frac {1} {R}\sum_{w\in\mathcal{R}} \E z_w$, $d:=\frac{W-1}{sW-1}$, and $\E$ is taken over all the randomness.
\end{lemma}

Without resampling, we have the decomposition
\begin{align}\label{decomposition}
	& \frac{1}{R}\sum_{w\in\mathcal{R}} \E \|z_w-\bar z\|^2 \\
	= & \frac{1}{R}\sum_{w\in\mathcal{R}}\E \left\|z_w - \E z_w\right\|^2
	+ \frac{1}{R}\sum_{w\in\mathcal{R}} \left\|\E z_w - \bar z\right\|^2. \nonumber
\end{align}
Let us compare the two equalities \eqref{lemma:resampling-inner-outer-variation-inequatlity} and \eqref{decomposition}. The left-hand sides of \eqref{lemma:resampling-inner-outer-variation-inequatlity} and \eqref{decomposition} are the mean-square errors of $\{\tilde{z}_{w'}, w' \in \mathcal{R}'\}$ and $\{z_w, w \in \mathcal{R}\}$ relative to $\bar z$, respectively. Since $d+\frac{1-d}{R} \leq 1$ and $d \leq 1$, we see that the bias of $\{\tilde{z}_{w'}, w' \in \mathcal{R}'\}$ to $\bar z$ is smaller than the bias of $\{z_w, w \in \mathcal{R}\}$ to $\bar z$, showing the ``variance reduction'' property of resampling.

\noindent \textbf{Remark 2}.
The proof of Lemma \ref{lemma:resampling-inner-outer-variation} is nontrivial. We handle the variance of resampled vectors from two aspects, namely, the mean-square error between resampled vectors and original vectors and that between original vectors and averaged vector. This enables the application of \cite[Proposition 1]{he2020byzantine} to the proof.

\subsection{Main theorem}
The following main theorem establishes the convergence property of the proposed method.
\begin{theorem}
	\label{Theorem:RS+Geomed+SAGA}
	Under Assumptions \ref{assumption1} and \ref{assumption2}, if the number of Byzantine workers satisfies $B < \frac{W}{2s}$ and the step size satisfies
	$$
	\gamma \leq \frac{\mu}{2\sqrt{10} J^2 L^2 C_{s\alpha} },
	$$
	then the iterate $x^k$ generated by the proposed method in Algorithm 2 satisfies
	%QL: then the iterate $x^k$ generated by the proposed method in Algorithm 2 with $\epsilon$-approximate geometric median aggregation satisfies
	%
	\begin{align}\label{th1-1}
		\E \|x^k - x^*\|^2 \leq (1 - \frac{\gamma\mu}{2})^k \Delta_1 + \Delta_2,
	\end{align}
	where
	\begin{align}\label{th1-2}
		\Delta_1 := \|x^0 - x^*\|^2 - \Delta_2,
	\end{align}
	\begin{align}\label{th1-3}
		\Delta_2 := \frac{5d C_{s\alpha}^2 \delta^2}{\mu^2},
	\end{align}
	while $\alpha :=\frac{B}{W}$, $C_{s\alpha}:= \frac{2-2s\alpha}{1-2s\alpha}$, $d: = \frac{W-1}{s W - 1}$, and $\E$ denote the expectation with respect to all random variables $i_w^k$ and the resampling processes.
\end{theorem}

Theorem \ref{Theorem:RS+Geomed+SAGA} shows that the proposed method reaches a neighborhood of the optimal solution at a linear convergence rate. The asymptotic learning error $\Delta_2={\cal O}(d C_{s\alpha}^2 \delta^2)$ is determined by the outer variation, rather than the inner variation. In contrast, the learning error of Byrd-SAGA is ${\cal O}(C_{\alpha}^2 \delta^2)$ when $B < \frac{W}{2}$ . For the combination of distributed SGD and geometric median (that we call as Byrd-SGD) and the combination of distributed SGD, resampling and geometric median (that we call as RS-Byrd-SGD and give a novel analysis in Theorem \ref{Theorem:RS-Byrd-SGD}), the learning errors are ${\cal O}(C_{\alpha}^2 \sigma^2 + C_{\alpha}^2 \delta^2 )$  and ${\cal O}((d + \frac{1-d}{R}) C_{s\alpha}^2 \sigma^2 + d C_{s\alpha}^2 \delta^2)$ when $B < \frac{W}{2}$ and $B < \frac{W}{2s}$, respectively. We compare the learning errors in Table 1.

\begin{table}[!h]
	\centering
	\begin{tabular}{cccc}
		\hline
		Algorithm & Requirement & Learning Error \\
		\hline
		Byrd-SGD & $B<\frac{W}{2}$ & ${\cal O}(C_{\alpha}^2 \sigma^2 + C_{\alpha}^2 \delta^2 )$ \cite{byrdsaga-arxiv} \\
		\hline
		RS-Byrd-SGD & $B<\frac{W}{2s}$ & ${\cal O}((d + \frac{1-d}{R}) C_{s\alpha}^2 \sigma^2 + d C_{s\alpha}^2 \delta^2)$ \\
		\hline
		Byrd-SAGA & $B<\frac{W}{2}$ & ${\cal O}(C_{\alpha}^2 \delta^2)$ \cite{byrdsaga-arxiv} \\
		\hline
		\textbf{Our Proposed} & $B<\frac{W}{2s}$ & ${\cal O}(d C_{s\alpha}^2 \delta^2)$\\
		\hline
	\end{tabular}
	\caption{Requirements \& learning errors of four algorithms.}
	\label{Table:asymptoticLearningError}
	\vspace{-2em}
\end{table}

\begin{figure}
	\centering
	\includegraphics[scale=0.45]{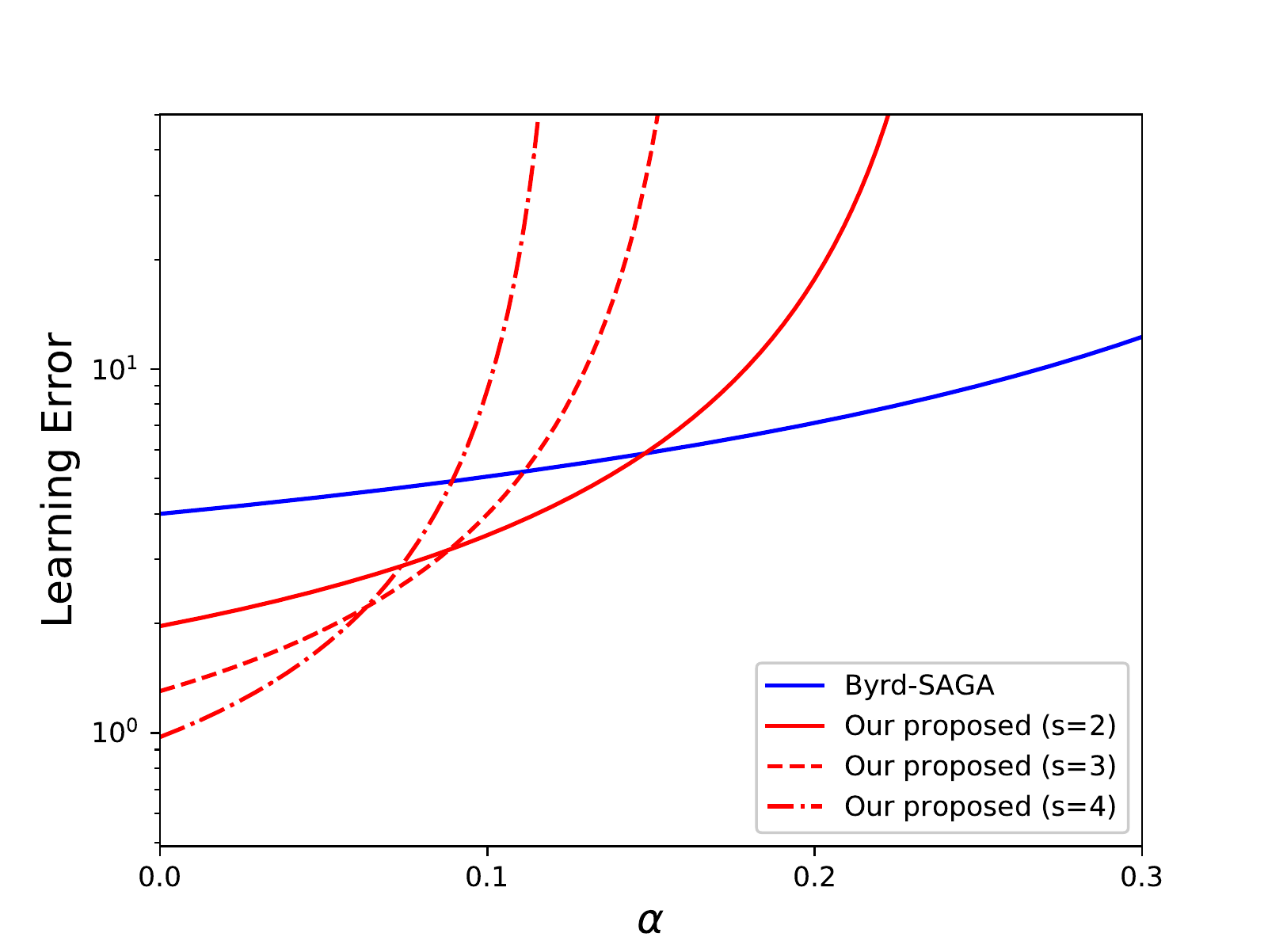}
	\caption{Learning error versus $\alpha$, the fraction of Byzantine workers, for Byrd-SAGA and the proposed method with different $s$. We omit $\delta^2$ in the learning errors to compare $C_\alpha^2$ with $d C_{s\alpha}^2$. The number of all workers is $W=30$.}
	\label{Fig:AsymptoticLearningError}
	\vspace{-1em}
\end{figure}

Fig. \ref{Fig:AsymptoticLearningError} shows the dependency of learning error on $\alpha := \frac{B}{W}$ for Byrd-SAGA and the proposed method with different $s$. We omit the outer variation $\delta^2$ in the learning errors and compare $C_\alpha^2$ with $d C_{s\alpha}^2$. The number of all workers is set as $W=30$, which is consistent with the setting in the numerical experiments. Observe that when $\alpha$ is small enough, the proposed method has smaller learning error than Byrd-SAGA. This fact depicts the advantage of resampling in reducing the outer variation and handling the non-i.i.d. data. The side-effect of resampling is the tolerance to a smaller fraction of Byzantine workers. Larger $s$ means tolerance to less Byzantine workers. According to the analysis, $s=2$ achieves satisfactory tradeoff between the learning error and the tolerable fraction of Byzantine workers. Therefore, in the numerical experiments we set $s=2$, which is also recommended by \cite{he2020byzantine}.

Note that \eqref{eq:geomed} has no closed-form solution, and must be solved with iterative algorithms. Therefore, we often consider $\epsilon$-approximate geometric median, allowing the computed vector and the true geometric median has a gap of $\epsilon$. We use a fast Weiszfeld's algorithm \cite{weiszfeld2009point} to compute the $\epsilon$-approximate geometric median in the numerical experiments. For simplicity, in the analysis above we assume $\epsilon=0$. In the appendices, we also consider the case of $\epsilon \neq 0$.

\noindent \textbf{Remark 3}. The proof of Theorem \ref{Theorem:RS+Geomed+SAGA} is built upon Lemma \ref{lemma:resampling-geomed}. There is a corresponding \cite[Theorem 1]{byrdsaga-arxiv} to show the convergence of Byrd-SAGA without resampling. The proof of \cite[Theorem 1]{byrdsaga-arxiv} is based on \cite[Lemma 1]{byrdsaga-arxiv}, whose difference with our Lemma \ref{lemma:resampling-geomed} has been discussed in Section \ref{sec:concentration}.

% QL: citation of geometric median calculation

\section{Numerical Experiments}
This section presents numerical experiments to demonstrate the robustness of our proposed method. We consider both convex and nonconvex distributed learning problems. For the convex problem, we focus on the softmax regression on the MNIST \cite{mnist} dataset. For the nonconvex problem, we train two-layer perceptrons, in which each layer has 50 neurons and the activation function is `tanh', on the MINST and COVTYPE \cite{covtype} datasets. The attributes of the datasets are described in Table 2. In the i.i.d. case, we use the MNIST dataset, launch 1 central node and $W=30$ workers, and let the data evenly distributed across all workers. In the non-i.i.d. case, we first use the MNIST dataset, launch 1 central node and $W=30$ workers, and let every three workers to evenly share the data from one class. We then use the COVTYPE dataset, launch 1 central node and $W=21$ workers, and also let every three workers to evenly share the data from one class. The numerical experiments are conducted on a server with two Intel(R) Xeon(R) Silver 4216 CPUs and four GeForce RTX 2080 GPUs. The source codes are available at \url{https://github.com/pengj97/Byzantine-robust-variance-reduction}

\begin{table}[tbp]
	\centering
	\label{table1}
	\setlength{\tabcolsep}{1.5mm}{
		\begin{tabular}{ccccc}
			\hline
			& & &   \\[-9pt]
			Name  &Train&  Test& Dimensions & Classes \\
			\hline
			& & &  \\[-9pt]
			MNIST  & 60000 &  10000 & 784 & 10 \\
			\hline
			& & &  \\[-9pt]
			COVTYPE  & 11340 &  565892 & 54 & 7\\
			\hline
		\end{tabular}
	}
	\caption{Attributes of the MNIST and COVTYPE datasets.}	\vspace{-2em}
\end{table}

\subsection{Benchmark methods}
We compare our proposed method with several benchmarks.

\noindent\textbf{Distributed SGD.} The distributed SGD aggregates the received messages by returning the mean, and hence has no robustness against Byzantine attacks.

\noindent\textbf{Byrd-SGD.} Byrd-SGD aggregates the received messages by returning the geometric median, as shown in \eqref{eq:geomed_update}.

\noindent\textbf{RS-Byrd-SGD.} RS-Byrd-SGD first resamples the received messages with $s$-replacement, and then aggregates the results by returning the geometric median.

\noindent\textbf{Krum.} Krum aggregates the received messages by returning the one that has the smallest summation of squared distances to its $W -B -2$ nearest neighbors, given by
\begin{align}
	& \mathrm{Krum}(\{z_w, w \in \mathcal{W}\}) = z_{w^*}, \nonumber \\
	& w^* = \arg\min_{w \in \mathcal{W}} \sum_{w \rightarrow w'} \|z_w - z_{w'}\|^2. \nonumber
\end{align}
Here $w \rightarrow w'$ selects the indexes $w'$ of the $W-B-2$ nearest neighbors of $z_w$ in $\{z_w, w \in \mathcal{W}\}$. Note that Krum needs to know the number of Byzantine workers in advance \cite{krumm-krum}.

\noindent\textbf{Byrd-SAGA.} Byrd-SAGA aggregates the received messages by returning the geometric median, but the regular workers returns the corrected stochastic gradients, not the stochastic gradients in the above benchmark methods \cite{byrdsaga,byrdsaga-arxiv}.

\noindent\textbf{RSA.} RSA is based on model aggregation, not stochastic gradient aggregation in the above benchmark methods. The central node and every worker $w$ maintains iterates $x_0^k$ and $x_w^k$, respectively. The update rule of the central node is
\begin{align}
	x_0^{k+1} = x_0^k - \gamma \cdot \lambda \sum_{w \in \mathcal{W}} \mathrm{sign}(x_0^k - x_w^k), \nonumber
\end{align}
while the update rule of regular worker $w \in \mathcal{R}$ is
\begin{align}
	x_w^{k+1} = x_w^k - \gamma \cdot \left( f'_{w, i_w^k} (x_w^k) + \lambda \mathrm{sign}(x_w^k - x_0^k) \right). \nonumber
\end{align}
Here $\mathrm{sign}$ is the element-wise sign function and $\lambda > 0$ is the penalty parameter \cite{RSA}.

In the numerical experiments, the batch size is 32. When using resampling, we set $s=2$ by default. All parameters in the benchmark methods are hand-tuned to the best.

\subsection{Byzantine attacks}

We consider the following Byzantine attacks.

\noindent \textbf{Sign-flipping attacks.} Every Byzantine worker $w \in \mathcal{B}$ computes its true message $\hat{v}_w^k$, and then sends $v_w^k = c\hat{v}_w^k$ to the central node. Here we set $c=-5$.

\noindent \textbf{Gaussian attacks.} Every Byzantine worker $w \in \mathcal{B}$ sends message $v_w^k$, whose elements follow Gaussian distribution $\mathcal{N}(0,10000)$, to the central node.

\noindent \textbf{Sample-duplicating attacks.} The Byzantine workers collude to choose a specific regular worker, and duplicate its message at every time. This amounts to that the Byzantine workers duplicate the data samples of the chosen regular worker. The sample-duplicating attacks are applied in the non-i.i.d. case.

\subsection{Softmax regression in the i.i.d. case}

We carry out numerical experiments in softmax regression on the MNIST dataset. The step size of our proposed method is $\gamma = 0.5$. When there exist Byzantine attacks, we uniformly randomly select $B=6$ Byzantine workers.

\begin{figure}
	\centering
	\includegraphics[width=6.2cm]{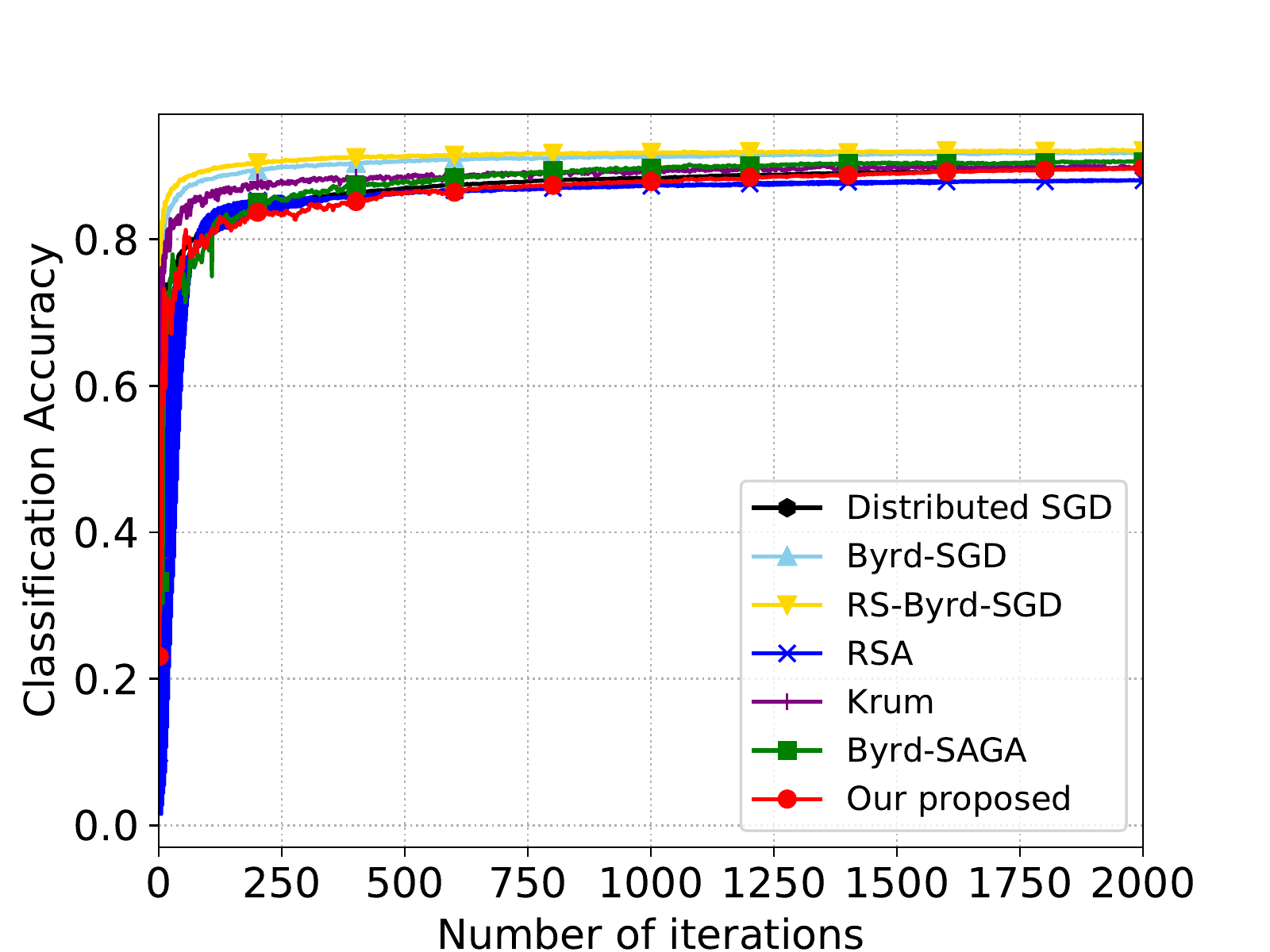}
	\caption{Without Byzantine attacks on i.i.d. MNIST data in softmax regression.}
	\label{WA}
\end{figure}
\begin{figure}
	\centering
	\includegraphics[width=6.2cm]{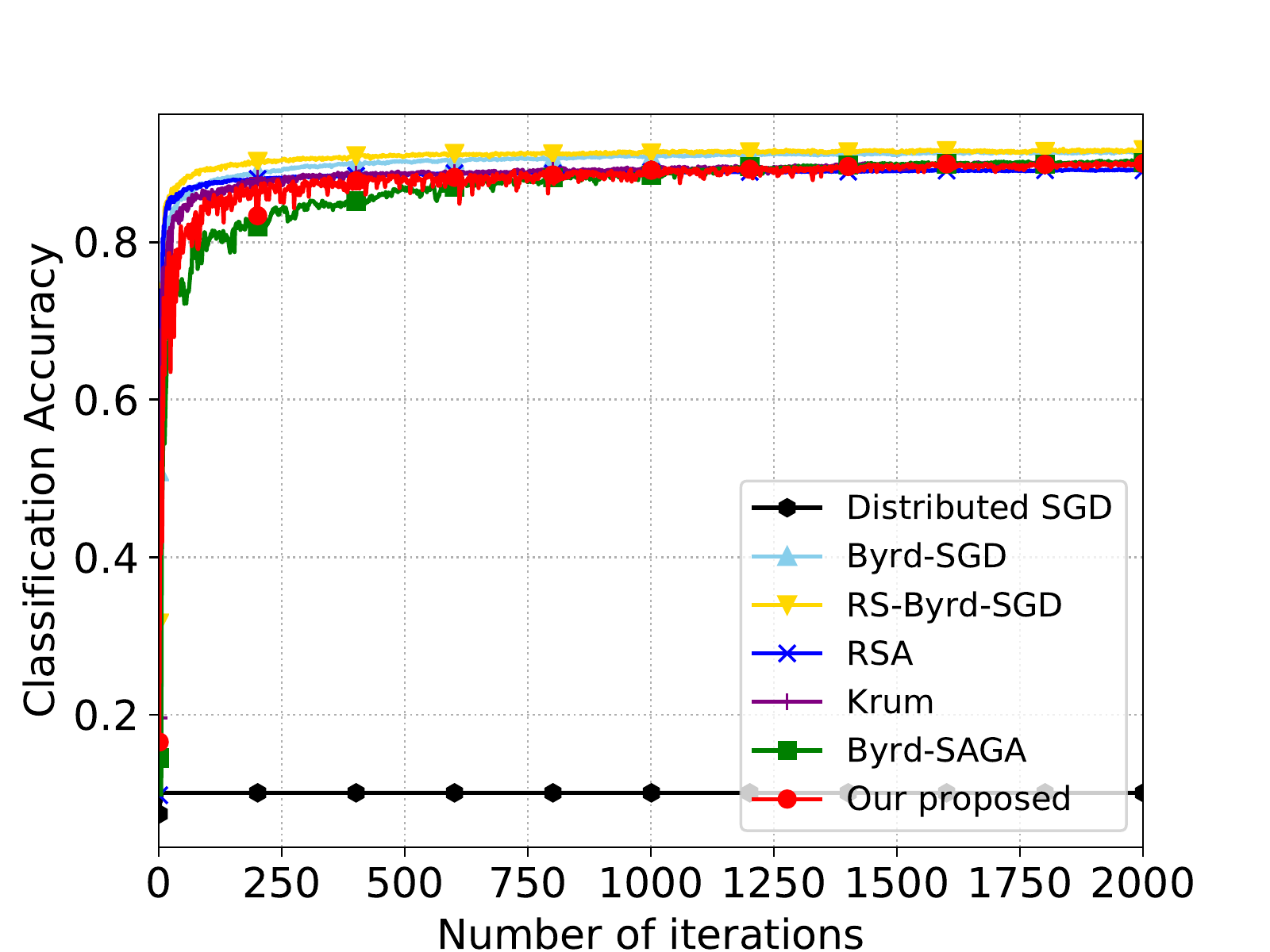}
	\caption{With sign-flipping attacks on i.i.d. MNIST data in softmax regression.}
	\label{SF}
\end{figure}
\begin{figure}
	\centering
	\includegraphics[width=6.2cm]{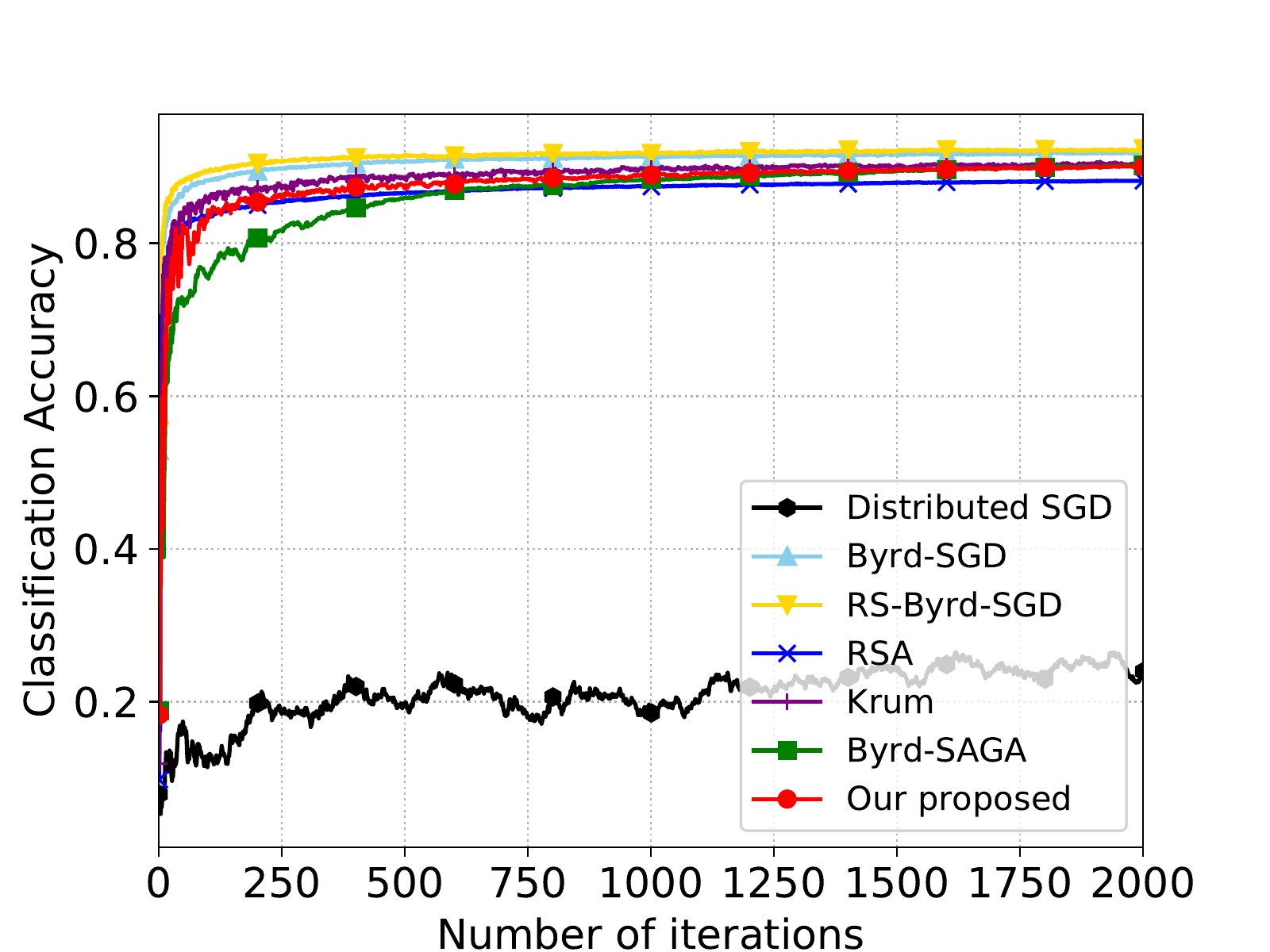}
	\caption{With Gaussian attacks on i.i.d. MNIST data in softmax regression.}
	\label{GA}
\end{figure}
\begin{figure}
	\centering
	\includegraphics[width=6.2cm]{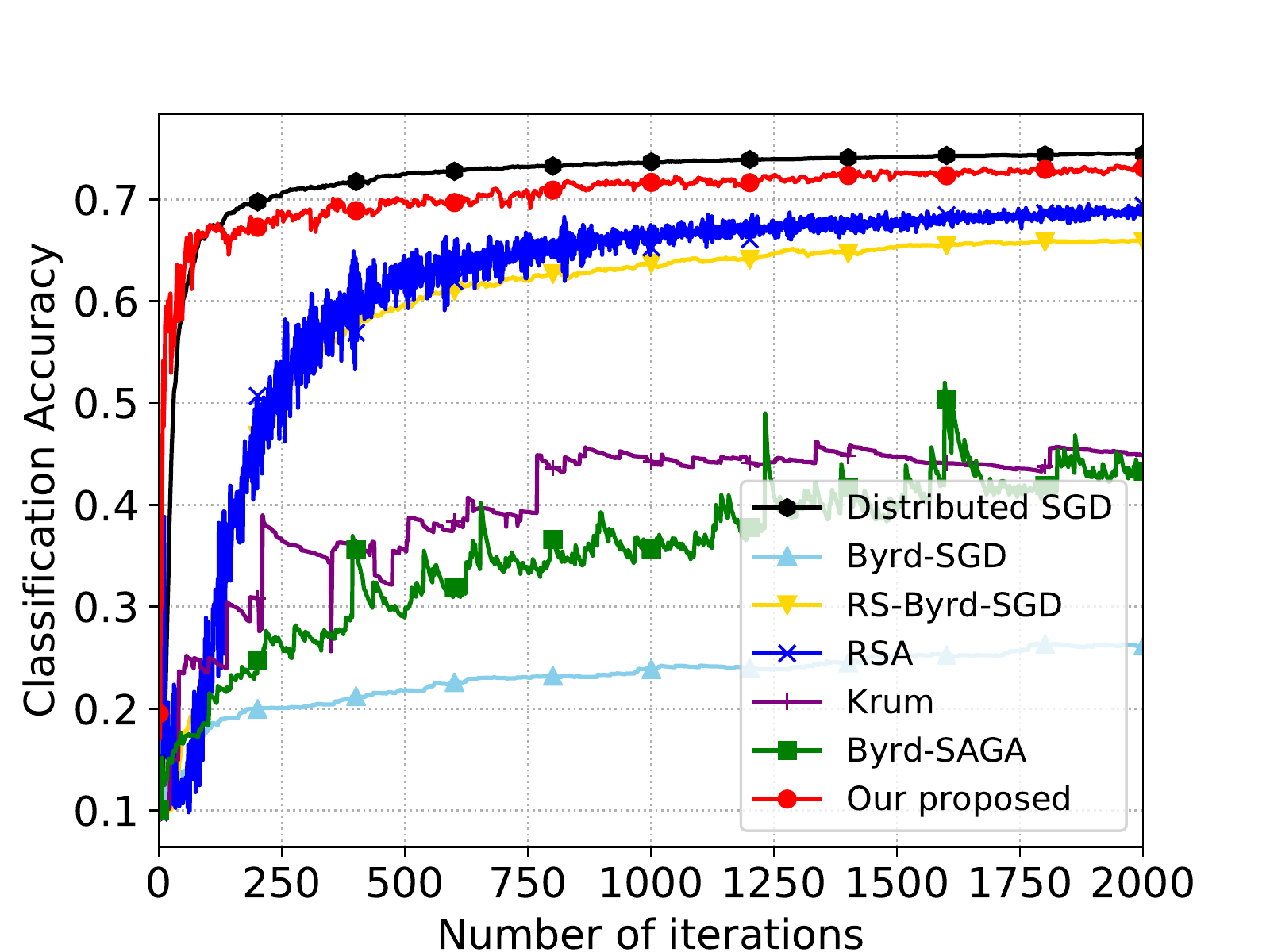}
	\caption{With sample-duplicating attacks on non-i.i.d. MNIST data in softmax regression.}
	\label{HD}
\end{figure}
\begin{figure}
	\centering
	\includegraphics[width=6.2cm]{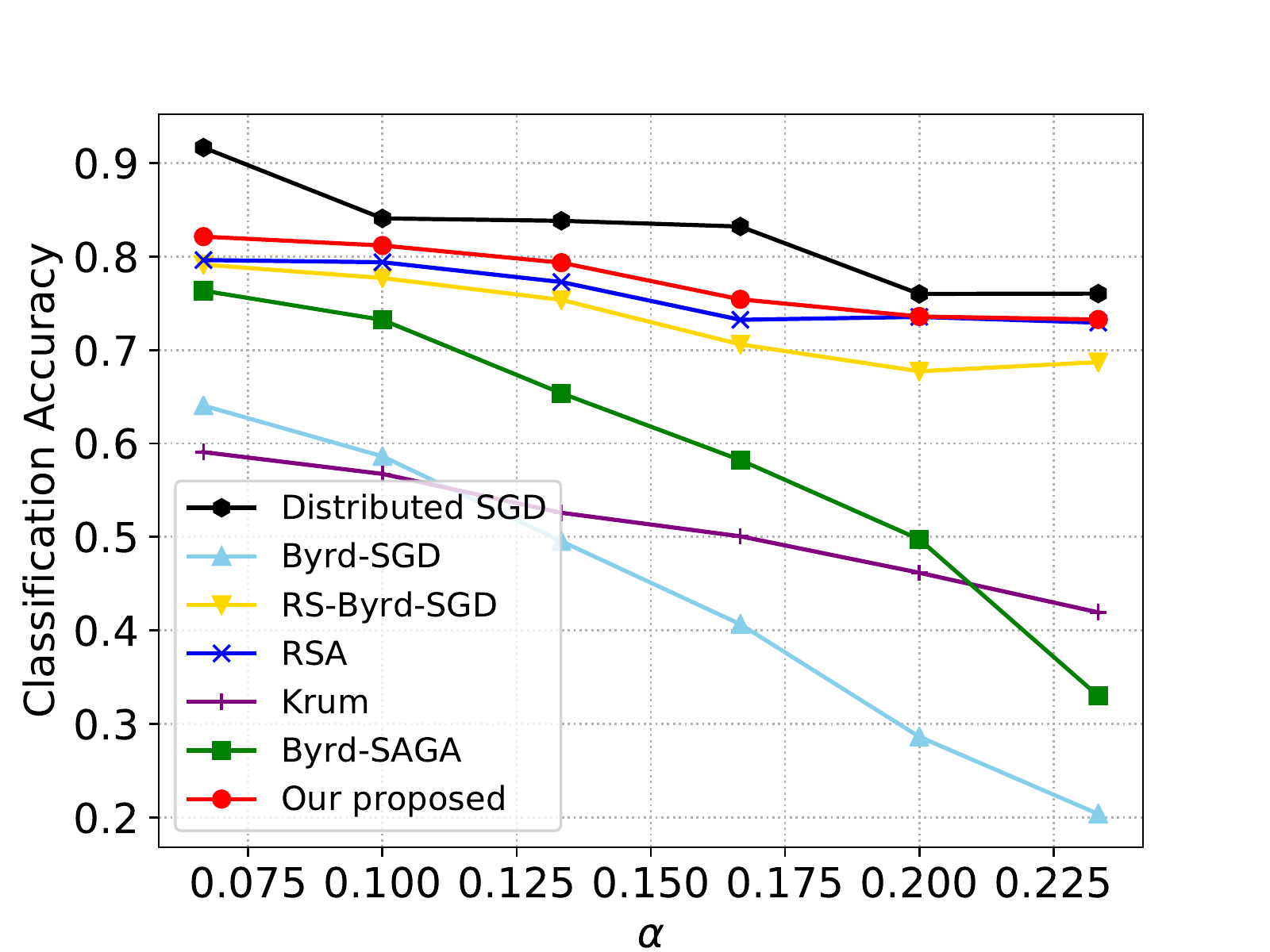}
	\caption{Impact of the fraction of Byzantine workers, with sample-duplicating attacks on non-i.i.d. MNIST data in softmax regression.}
	\label{function_alpha}
\end{figure}
\begin{figure}
	\centering
	\includegraphics[width=6.2cm]{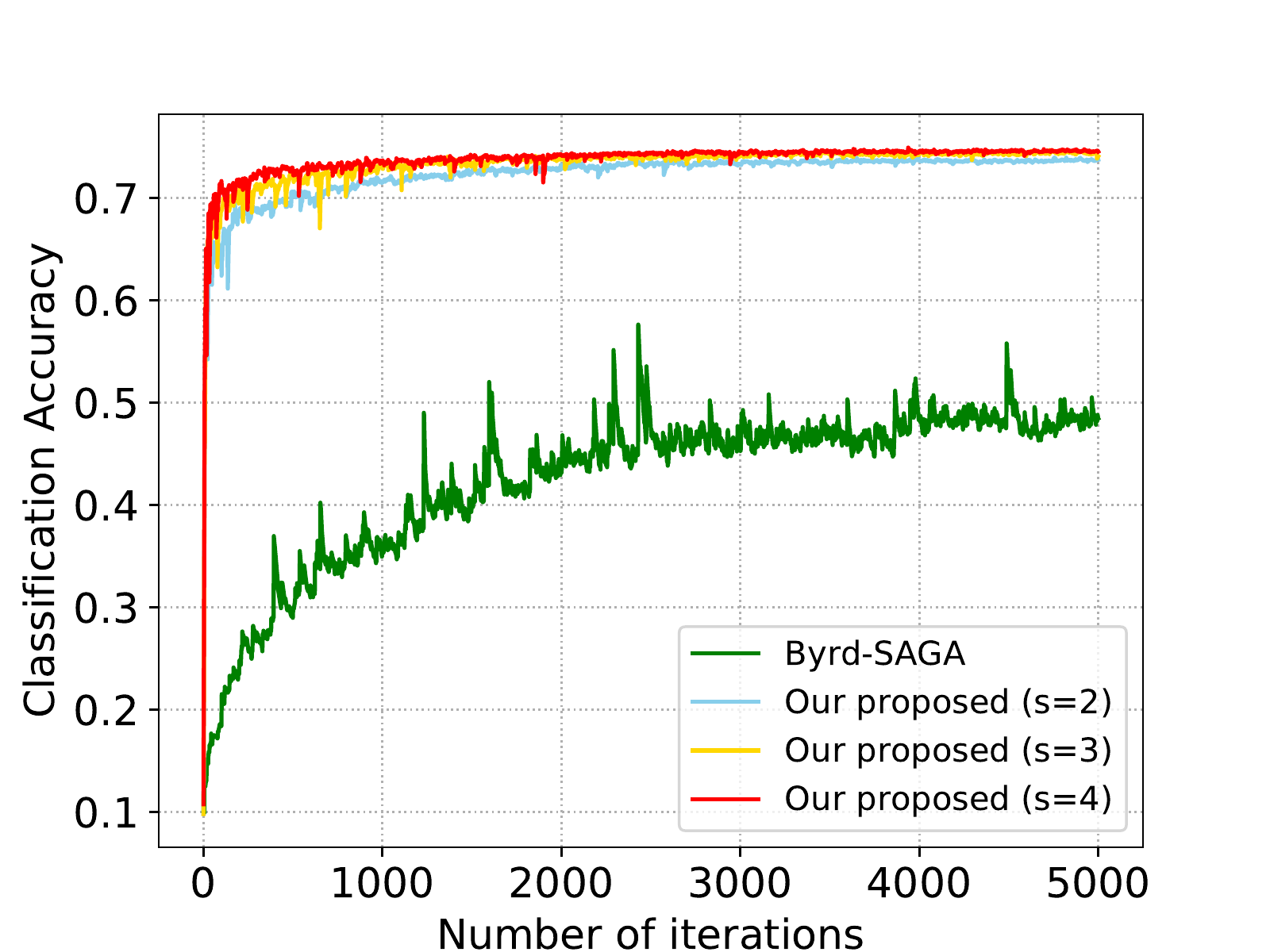}
	\caption{Impact of the resampling size, with sample-duplicating attacks on non-i.i.d. MNIST data in softmax regression.}
	\label{function_s}
\end{figure}

\noindent \textbf{Without Byzantine attacks.} When there exist no Byzantine attacks, all the methods are able to achieve satisfactory classification accuracies, as illustrated in Fig. \ref{WA}.

\noindent \textbf{Sign-flipping attacks.} As depicted in Fig. \ref{SF}, the distributed SGD is vulnerable and fails under the sign-flipping attacks. The other methods all perform well.

\noindent \textbf{Gaussian attacks.} The results of Gaussian attacks are shown in Fig. \ref{GA}. The distributed SGD also fails, while the other methods have robustness to Gaussian attacks.

\subsection{Softmax regression in the non-i.i.d. case}

\noindent \textbf{Sample-duplicating attacks.} We choose $B=6$ workers that originally share the samples of two classes as Byzantine workers. Therefore, these two classes essentially disappear under the sample-duplicating attacks, such that the best possible accuracy is 0.8. For the regular worker that is chosen to be duplicated, its class now has 9 workers out of the total $W=30$. The step size of our proposed method is $\gamma = 0.5$. As shown in Fig. \ref{HD}, our proposed method is able to achieve the classification accuracy of around 0.73. RS-Byrd-SGD also demonstrates robustness since the resampling strategy helps reduce the impact of both inner and outer variations. However, as we have shown in the theoretical analysis, solely using resampling is unable to fully eliminate the inner variation. In contrast, the proposed method introduces SAGA to address this issue, yielding enhanced classification accuracy. RSA works well since it is developed to handle the non-i.i.d. case. The other Byzantine-robust methods, Byrd-SGD, Krum and Byrd-SAGA, all show degraded performance since they are essentially designed for distributed i.i.d. data.

\noindent \textbf{Impact of the fraction of Byzantine workers $\alpha$.} Theorem \ref{Theorem:RS+Geomed+SAGA} and Fig. \ref{Fig:AsymptoticLearningError} demonstrate that the learning error of the proposed method is monotonically increasing when the fraction of Byzantine workers $\alpha$ increases. This theoretical result is validated by Fig. \ref{function_alpha}. The other experimental settings are the same as those in Fig. \ref{HD}. The classification accuracies of all the methods degrade when $\alpha$ increases. Compared to other benchmark methods, our proposed method is comparatively insensitive to the variation of $\alpha$.

\noindent \textbf{Impact of the resampling size $s$.} Theorem \ref{Theorem:RS+Geomed+SAGA} and Fig. \ref{Fig:AsymptoticLearningError} also shows that the learning error of the proposed method is dependent on the resampling size $s$. Fig. \ref{function_s} depicts how the classification accuracies vary with respect to different $s$. The other experimental settings are the same as those in Fig. \ref{HD}. When $s=1$, our proposed method degenerates to Byrd-SAGA, which is not Byzantine-robust in the non-i.i.d. case. When $s=3$ and $s=4$, for these particular attacks, the proposed method still performs well, although the fraction of Byzantine workers $\alpha = \frac{1}{5}$ exceeds the theoretically tolerable bound $\frac{1}{2s}$.

\subsection{Neural network training in the non-i.i.d. case}

\noindent \textbf{Sample-duplicating attacks on MNIST.} We choose $B=3$ workers that originally share the samples of one class as Byzantine workers, such that the best possible accuracy is 0.9 under the sample-duplicating attacks. The step size of our proposed method is $\gamma = 0.1$. We compare Byrd-SGD, RS-Byrd-SGD, Byrd-SAGA and our proposed method in Fig. \ref{nn_mnist}. Byrd-SGD fails, but Byrd-SAGA is able to attain favorable classification accuracy since the ratio of Byzantine workers is only $10\%$. Among the two methods developed for the non-i.i.d. case, our proposed method outperforms RS-Byrd-SGD. The performance gain of Byrd-SAGA over Byrd-SGD and that of the proposed method over RS-Byrd-SGD are both due to the elimination of inner variation.

\noindent \textbf{Sample-duplicating attacks on COVTYPE.} We choose $B=3$ workers that originally share the samples of one class as Byzantine workers, such that the best possible accuracy is 0.86 under the sample-duplicating attacks. The step size of our proposed method is $\gamma = 0.1$. As shown in Fig. \ref{nn_covtype}, Byrd-SGD is still the worst. However, Byrd-SAGA and RS-Byrd-SGD are both remarkably outperformed by our proposed method, since the ratio of Byzantine workers is now raised to $14\%$. Therefore, eliminating the inner variation and reducing the outer variation are of particular importance as a larger portion of received messages are Byzantine.

\begin{figure}
	\centering
	\includegraphics[width=6.2cm]{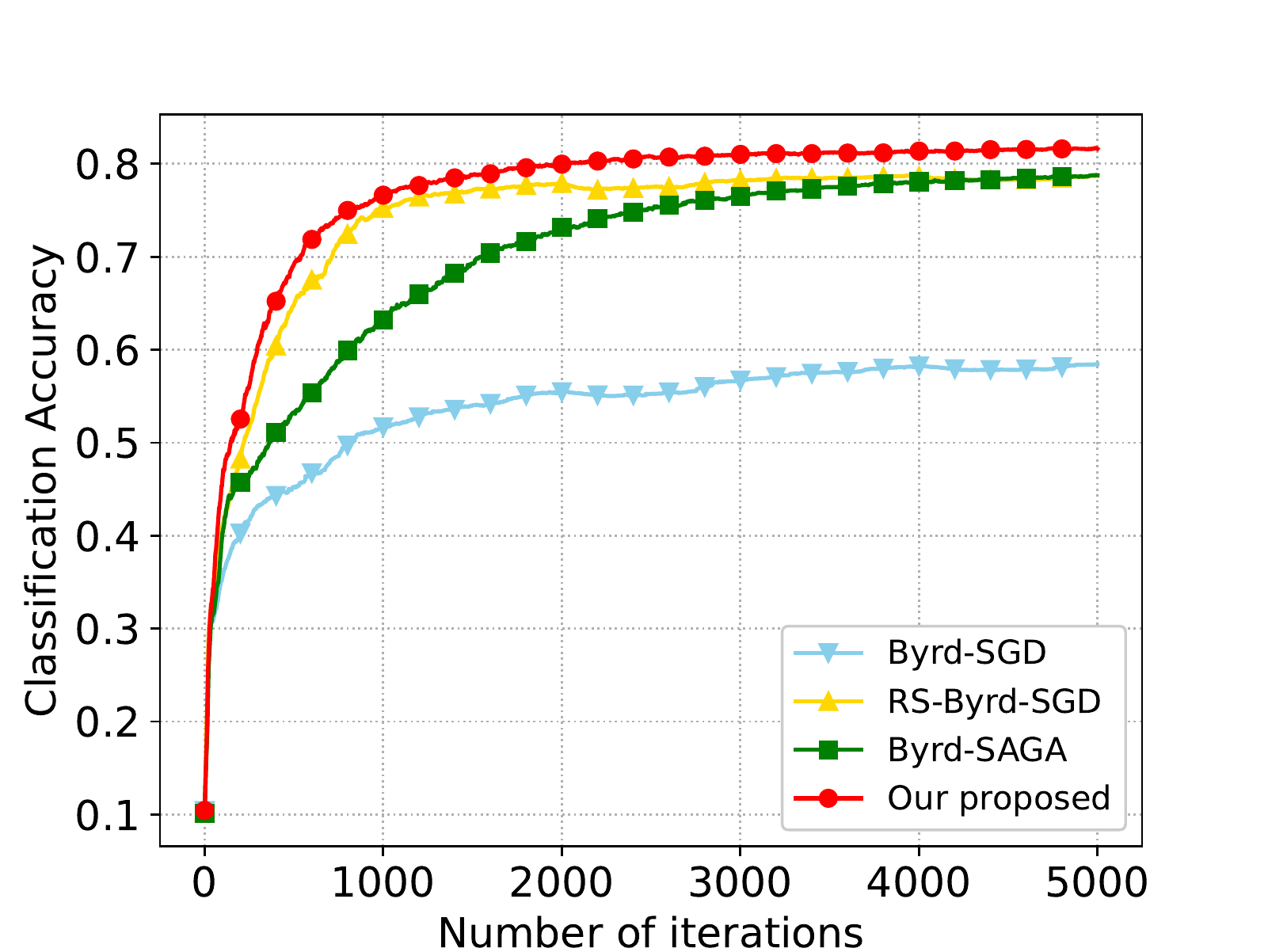}
	\caption{Classification accuracy with sample-duplicating attacks on non-i.i.d. MNIST data in neural network training.}
	\label{nn_mnist}
\end{figure}
\begin{figure}
	\centering
	\includegraphics[width=6.2cm]{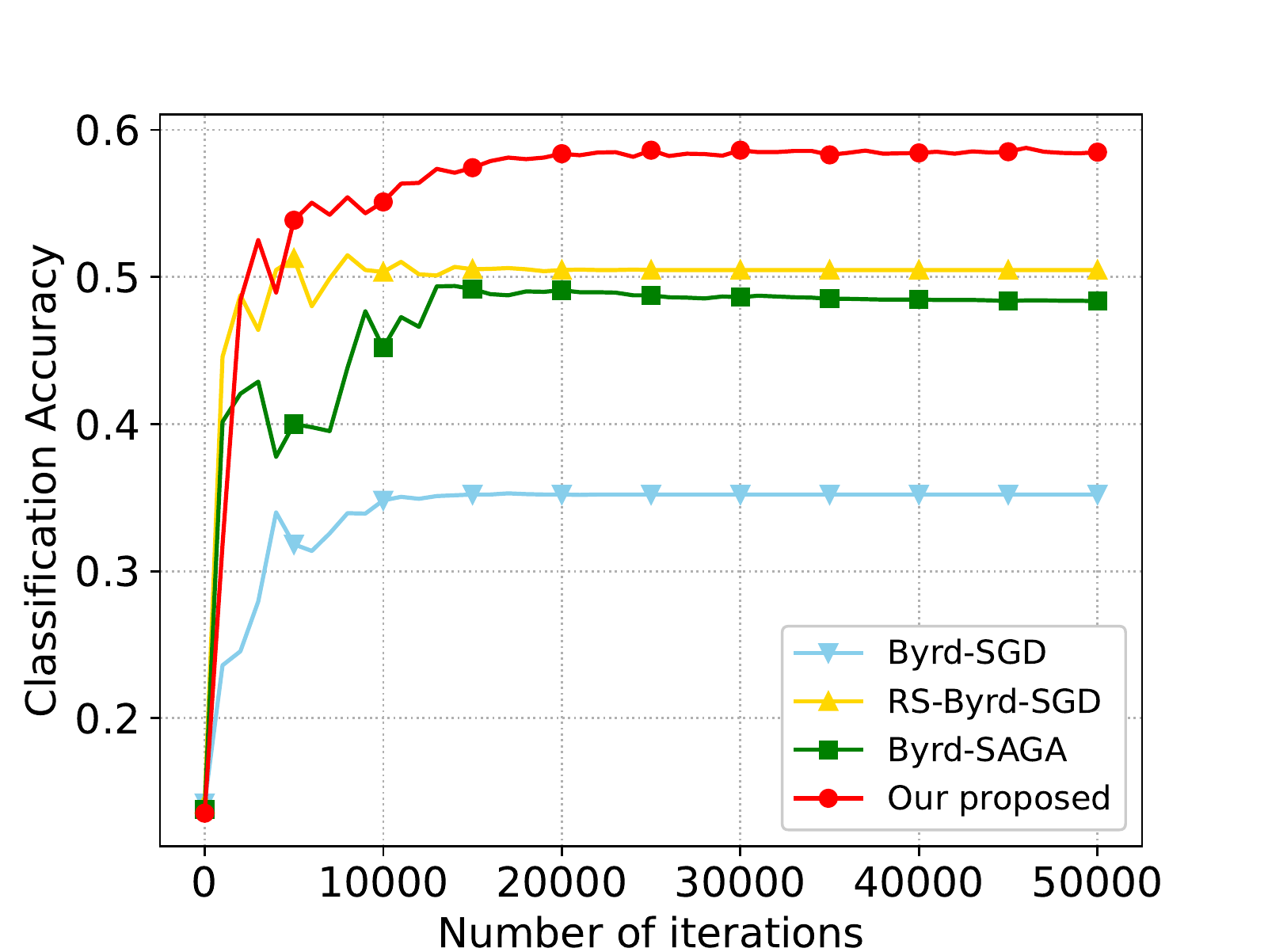}
	\caption{Classification accuracy with sample-duplicating attacks on non-i.i.d. COVTYPE data in neural network training.}
	\label{nn_covtype}
\end{figure}

\section{Conclusions}
We develop a Byzantine-robust variance-reduced method to deal with the federated learning problem with distributed non-i.i.d. data. To reduce the impact of stochastic gradient noise that hinders the resistance to Byzantine attacks, we adopt the resampling strategy and SAGA to reduce the outer variation and fully eliminate the inner variation. The variance-reduced messages are then aggregated by geometric median. Theoretical results show that the proposed method can reach a neighborhood of the optimal solution with linear convergence rate and the learning error is determined by the number of Byzantine workers. Numerical experiments demonstrate the robustness of our proposed method. In the future work, we will investigate the elimination of outer variation, and the combination of variance-reduced methods with other robust aggregation rules.

% if have a single appendix:
%\appendix[Proof of the Zonklar Equations]
% or
%\appendix  % for no appendix heading
% do not use \section anymore after \appendix, only \section*
% is possibly needed

% use appendices with more than one appendix
% then use \section to start each appendix
% you must declare a \section before using any
% \subsection or using \label (\appendices by itself
% starts a section numbered zero.)
%

\appendices
\vspace{1em}

\section{Proof of Lemma \ref{lemma:resampling-geomed}}

We begin with reviewing the supporting lemma.
\begin{lemma}
	\label{lemma:resampling}
	\cite[Proposition 1]{he2020byzantine}
	Let $\{z_w, w\in\mathcal{W}\}$ be a set of vectors, and generate from it a new set $\{\tilde{z}_w, w\in \mathcal{W}\}$ using the resampling strategy with $s$-replacement. When $B < \frac{W}{s}$, there exists a set $\mathcal{R}' \subseteq \mathcal{W}$ with at least $W-sB$ elements, such that for any $w' \in \mathcal{R}'$ it holds that
	\begin{align}
		\E \tilde{z}_{w'} = \frac{1}{R}\sum_{w\in\mathcal{R}} z_w,
		\label{equation:resampling-expectation}
	\end{align}
	\begin{align}
		\E\|\tilde{z}_{w'}-\E \tilde{z}_{w'}\|^2 =
		\frac{d}{R}\sum_{w\in\mathcal{R}}
		\left\|z_w-\frac{1}{R}\sum_{u\in\mathcal{R}}z_{u} \right\|^2,
		\label{equation:resampling-variance}
	\end{align}
	where $d:=\frac{W-1}{sW-1}$ and the expectation $\E$ is taken over the resampling process.
\end{lemma}

Lemma \ref{lemma:resampling} shows that after resampling, the expectations of at least $W-sB$ elements are the average of $\{z_w, w\in \mathcal{R}\}$. Their variance $\E\|\tilde{z}_{w'}-\E \tilde{z}_{w'}\|^2$ reduces by a factor of $\frac{1}{d}$ relative to the variation $\frac{1}{R}\sum_{w\in\mathcal{R}} \|z_w-\frac{1}{R}\sum_{w\in\mathcal{R}}z_{w}\|^2$. Note that $\frac{1}{d}$ is close to $s$ if $W$ is sufficiently large. The properties of these $W-sB$ elements are further investigated in Lemma \ref{lemma:resampling-inner-outer-variation}, whose proof is as follows.
%\vspace{-20pt}
\begin{proof}
	We decompose $\frac{1}{|\mathcal{R'}|}\sum_{w'\in\mathcal{R}'} \E \|\tilde{z}_{w'}-\bar z\|^2$ into two terms
	\begin{align}
		\label{inequality:resample-2}
		& \frac{1}{|\mathcal{R'}|}\sum_{w'\in\mathcal{R}'} \E \|\tilde{z}_{w'}-\bar z\|^2 \\
		% -------------------
		=& \frac{1}{|\mathcal{R'}|}\sum_{w'\in\mathcal{R}'}
		\E \left\|
		\tilde{z}_{w'} - \frac{1}{R}\sum_{w\in \mathcal{R}}z_{w}
		\right\|^2
		+ \E \left\|
		\frac{1}{R}\sum_{w\in \mathcal{R}}z_{w}-\bar z
		\right\|^2 \nonumber \\
		% -------------------
		= & \frac{d}{R}\sum_{w\in\mathcal{R}}\E \left\|
		z_w - \frac{1}{R}\sum_{u\in \mathcal{R}}z_{u}
		\right\|^2
		+ \E \left\|
		\frac{1}{R}\sum_{w\in \mathcal{R}}z_{w}-\bar z
		\right\|^2,  \nonumber
	\end{align}
	where the last equality comes from Lemma \ref{lemma:resampling} that for any $w' \in \mathcal{R}'$ it holds
	\begin{align}
		\E \left\|
		\tilde{z}_{w'} - \frac{1}{R}\sum_{w\in \mathcal{R}}z_{w}
		\right\|^2
		= & \E \left\|
		\tilde{z}_{w'} - \E \tilde{z}_{w'}
		\right\|^2 \nonumber\\
		= & \frac{d}{R}\sum_{w\in\mathcal{R}}\E \left\|
		z_w - \frac{1}{R}\sum_{u\in \mathcal{R}}z_{u}
		\right\|^2.\nonumber
	\end{align}
	The first term at the right-hand side of \eqref{inequality:resample-2} can be further decomposed into three terms
	\begin{align}
		\label{inequality:resample-3}
		&\frac{1}{R}\sum_{w\in\mathcal{R}}\E \left\|
		z_w - \frac{1}{R}\sum_{u\in \mathcal{R}}z_{u}
		\right\|^2  \\
		% -------------------
		=& \frac{1}{R}\sum_{w\in\mathcal{R}}\E \Bigg\|
		(z_w - \E z_w)
		+ \left(\E z_w - \frac{1}{R}\sum_{u\in \mathcal{R}}\E z_{u}\right) \nonumber \\
		&+ \left(\frac{1}{R}\sum_{u\in \mathcal{R}}\E z_{u} - \frac{1}{R}\sum_{u\in \mathcal{R}}z_{u}\right)
		\Bigg\|^2 \nonumber \\
		% -------------------
		=& \frac{1}{R}\sum_{w\in\mathcal{R}}
		\E \left\|z_w - \E z_w\right\|^2
		+ \frac{1}{R}\sum_{w\in\mathcal{R}} \left\|\E z_w - \frac{1}{R}\sum_{u\in \mathcal{R}}\E z_{u}\right\|^2 \nonumber \\
		+& \E \left\|\frac{1}{R}\sum_{w\in \mathcal{R}}z_{w} -\frac{1}{R}\sum_{w\in \mathcal{R}}\E z_{w} \right\|^2 \nonumber \\
		+& \underbrace{
			\frac{2}{R}\sum_{w\in\mathcal{R}} \E\left\langle
			z_w - \E z_w,
			\E z_w - \frac{1}{R}\sum_{u\in \mathcal{R}}\E z_{u}
			\right\rangle
		}_{T_1} \nonumber \\
		+& \underbrace{
			\frac{2}{R}\sum_{w\in\mathcal{R}} \E\left\langle
			\E z_w - \frac{1}{R}\sum_{u\in \mathcal{R}}\E z_{u},
			\frac{1}{R}\sum_{u\in \mathcal{R}}\E z_{u} - \frac{1}{R}\sum_{u\in \mathcal{R}}z_{u}
			\right\rangle
		}_{T_2} \nonumber \\
		+& \underbrace{
			\frac{2}{R}\sum_{w\in\mathcal{R}} \E\left\langle
			z_w - \E z_w,
			\frac{1}{R}\sum_{u\in \mathcal{R}}\E z_{u} - \frac{1}{R}\sum_{u\in \mathcal{R}}z_{u}
			\right\rangle
		}_{T_3} \nonumber \\
		% -------------------
		=& \frac{1}{R}\sum_{w\in\mathcal{R}}
		\E \left\|z_w - \E z_w\right\|^2 + \frac{1}{R}\sum_{w\in\mathcal{R}} \left\|\E z_w - \frac{1}{R}\sum_{u\in \mathcal{R}}\E z_{u}\right\|^2 \nonumber \\
		&- \E \left\|\frac{1}{R}\sum_{w\in \mathcal{R}}z_w -\frac{1}{R}\sum_{w\in \mathcal{R}}\E z_w  \right\|^2. \nonumber
	\end{align}
	To see how to reach the last equality, we check the three cross terms. The first cross term can be cancelled by
	\begin{align}
		T_1 =& \frac{2}{R}\sum_{w\in\mathcal{R}} \E\left\langle
		z_w - \E z_w,
		\E z_w - \frac{1}{R}\sum_{u\in \mathcal{R}}\E z_{u}
		\right\rangle  \\
		% -------------------
		=& \frac{2}{R}\sum_{w\in\mathcal{R}} \left\langle
		\underbrace{
			\E z_w - \E z_w
		}_{=0},
		\E z_w - \frac{1}{R}\sum_{u\in \mathcal{R}}\E z_{u}
		\right\rangle \nonumber \\
		% -------------------
		=& 0. \nonumber
	\end{align}
	Similarly, the second cross term can be cancelled by
	\begin{align}
		T_2 		=&\frac{2}{R}\sum_{w\in\mathcal{R}} \E\left\langle
		\E z_w - \frac{1}{R}\sum_{u\in \mathcal{R}}\E z_{u},
		\frac{1}{R}\sum_{u\in \mathcal{R}}\E z_{u} - \frac{1}{R}\sum_{u\in \mathcal{R}}z_{u}
		\right\rangle \\
		% -------------------
		=& \frac{2}{R}\sum_{w\in\mathcal{R}} \left\langle
		\E z_w - \frac{1}{R}\sum_{u\in \mathcal{R}}\E z_{u},
		\underbrace{
			\frac{1}{R}\sum_{u\in \mathcal{R}}\E z_{u} - \frac{1}{R}\sum_{u\in \mathcal{R}}\E z_{u}
		}_{=0}
		\right\rangle \nonumber \\
		% -------------------
		=& 0 \nonumber.
	\end{align}
	The third cross term equals to
	\begin{align}
		T_3 =& \frac{2}{R}\sum_{w\in\mathcal{R}} \E\left\langle
		z_w - \E z_w,
		\frac{1}{R}\sum_{u\in \mathcal{R}}\E z_{u} - \frac{1}{R}\sum_{u\in \mathcal{R}}z_{u}
		\right\rangle \\
		% -------------------
		=& 2 \E\left\langle
		\frac{1}{R}\sum_{w\in\mathcal{R}} z_w - \frac{1}{R}\sum_{w\in\mathcal{R}} \E z_w,
		\frac{1}{R}\sum_{u\in \mathcal{R}}\E z_{u} - \frac{1}{R}\sum_{u\in \mathcal{R}}z_{u}
		\right\rangle \nonumber \\
		% -------------------
		=& -2\E\left\|\frac{1}{R}\sum_{w\in\mathcal{R}} z_w - \frac{1}{R}\sum_{w\in\mathcal{R}} \E z_w \right\|^2 \nonumber.
	\end{align}
	
	Substituting \eqref{inequality:resample-3} into \eqref{inequality:resample-2} and using the definition $\bar z:=\frac{1}{R}\sum_{w\in \mathcal{R}}\E z_{w}$, we have
	\begin{align}
		\label{inequality:resample-4}
		&\frac{1}{|\mathcal{R'}|}\sum_{w'\in\mathcal{R}'} \E \|\tilde{z}_{w'}-\bar z\|^2 \\
		% -------------------
		%    =& \frac{1}{|\mathcal{R'}|}\sum_{w\in \mathcal{R}'}
		%    \left(
		%        \E \left\|
		%            \tilde{z}_w - \frac{1}{R}\sum_{w'\in \mathcal{R}}z_{w'}
		%        \right\|^2
		%        + \E \left\|
		%           \frac{1}{R}\sum_{w'\in \mathcal{R}}z_{w'}-\bar z
		%        \right\|^2
		%    \right) \nonumber \\
		% -------------------
		=&
		\frac{d}{R}\sum_{w\in\mathcal{R}}
		\E \left\|z_w - \E z_w\right\|^2
		+ \frac{d}{R}\sum_{w\in\mathcal{R}} \left\|\E z_w - \frac{1}{R}\sum_{u\in \mathcal{R}}\E z_{u}\right\|^2 \nonumber \\
		&+ (1-d) \E \left\|\frac{1}{R}\sum_{w\in \mathcal{R}}z_w - \frac{1}{R}\sum_{w\in \mathcal{R}}\E z_w\right\|^2
		\nonumber \\
		% -------------------
		=&
		\frac{d}{R}\sum_{w\in\mathcal{R}}
		\E \left\|z_w - \E z_w\right\|^2
		+ \frac{d}{R}\sum_{w\in\mathcal{R}} \left\|\E z_w - \bar z\right\|^2 \nonumber \\
		&+ \frac{1-d}{R^2}\sum_{w\in \mathcal{R}} \E \left\|z_w - \E z_w\right\|^2
		\nonumber \\
		% -------------------
		=&
		\left(d+\frac{1-d}{R}\right)\frac{1}{R}\sum_{w\in\mathcal{R}}
		\E \left\|z_w - \E z_w\right\|^2 \nonumber\\
		&+ \frac{d}{R}\sum_{w\in\mathcal{R}} \left\|\E z_w - \bar z\right\|^2
		. \nonumber
	\end{align}
	The second equality in \eqref{inequality:resample-4} comes from
	\begin{align}
		&\E \left\|\frac{1}{R}\sum_{w\in \mathcal{R}}z_w - \frac{1}{R}\sum_{w\in \mathcal{R}}\E z_w\right\|^2 \\
		=&\E \left\|\frac{1}{R}\sum_{w\in \mathcal{R}}(z_w - \E z_w)\right\|^2 \nonumber\\
		% -------------------
		=&\frac{1}{R^2}\sum_{w\in \mathcal{R}}\E \left\|z_w - \E z_w\right\|^2 \nonumber \\
		% -------------------
		&+ \sum_{w,u\in \mathcal{R}\atop w\neq u}\E\left\langle
		\frac{1}{R} (z_w - \E z_w),
		\frac{1}{R} (z_{u} - \E z_{u})
		\right\rangle\nonumber \\
		% -------------------
		=&\frac{1}{R^2}\sum_{w\in \mathcal{R}}\E \left\|z_w - \E z_w\right\|^2, \nonumber
	\end{align}
	where the last equality is due to the fact that the random vectors in $\{z_w, w\in \mathcal{R}\}$ are independent. This completes the proof.
\end{proof}

With Lemma \ref{lemma:resampling-inner-outer-variation}, the proof of Lemma \ref{lemma:resampling-geomed} is straightforward.

%	Now we compare the two bounds \eqref{decomposition} and \eqref{lemma:resampling-inner-outer-variation-inequatlity}. The left-hand sides of \eqref{decomposition} and \eqref{lemma:resampling-inner-outer-variation-inequatlity} are the mean-square errors of $\{z_w, w \in \mathcal{R}\}$ and $\{\tilde{z}_{w'}, w' \in \mathcal{R}'\}$ relative to $\bar z$, respectively. Since $d+\frac{1-d}{R} \leq 1$ and $d \leq 1$, we see that the bias of $\{\tilde{z}_{w'}, w' \in \mathcal{R}'\}$ to $\bar z$ is smaller than the bias of $\{z_w, w \in \mathcal{R}\}$ to $\bar z$, showing the ``variance reduction'' property of resampling.

%most of elements have the expectation of the original set's average, and enjoy roughly $1/s$ variance, which means they get closer to the average $\frac{1}{R}\sum_{w'\in\mathcal{R}}z_{w'}$. The question now is, how resampling operation has an effect on inner and out variation? Before the final answer, we first review the relation between total variation, inner and outer variation.
%
%We can see the first term is inner variation and the second one is outer variation. The following lemma shows that resampling has slightly different effect on inner and outer variation.

\begin{lemma}
	\label{lemma:resampling-geomed-full} (Full version of Lemma \ref{lemma:resampling-geomed})
	Let $\{z_w, w\in \mathcal{W}\}$ be a subset of random vectors distributed in a normed vector space and the random vectors in $\{z_w, w\in \mathcal{R}\}$ are independent. Generate from $\{z_w, w\in \mathcal{W}\}$ a new set $\{\tilde{z}_w, w\in \mathcal{W}\}$ using the resampling strategy with $s$-replacement. It holds when $B < \frac{W}{2s}$ that
	\begin{align}\label{lm1-1-full}
		&\E \| \mathrm{geomed}\left(\{\tilde{z}_w, w\in\mathcal{W}\} \right) - \bar{z}\|^2\  \\
		% -------------------
		\le &\left(d+\frac{1-d}{R}\right)\frac{C_{s\alpha}^2}{R} \sum_{w\in\mathcal{R}}{ \E\|z_w-\E z_w\|^2 } \nonumber\\
		&+ \frac{dC_{s\alpha}^2}{R} \sum_{w\in\mathcal{R}}{ \|\E z_w- \bar{z}\|^2}, \nonumber
	\end{align}
	where $\bar{z}:=\frac {1} {R}\sum_{w\in\mathcal{R}} \E z_w$, $\alpha :=\frac{B}{W}$, $C_{s\alpha}:= \frac{2-2s\alpha}{1-2s\alpha}$, $d: = \frac{W-1}{s W - 1}$, and $\E$ is taken over the random vectors and the resampling process. Define $\tilde{z}_\epsilon^*$ as an $\epsilon$-approximate geometric median of $\{\tilde{z}_w, w\in\mathcal{W}\}$. It holds when $B < \frac{W}{2s}$ that
	\begin{align}\label{lm1-1-noise}
		&\E \| \tilde{z}_\epsilon^* - \bar{z}\|^2  \\
		% -------------------
		\leq &\left(d+\frac{1-d}{R}\right)\frac{2C_{s\alpha}^2}{R} \sum_{w\in\mathcal{R}}{ \E\|z_w-\E z_w\|^2 } \nonumber \\
		% -------------------
		&+ \frac{2dC_{s\alpha}^2}{R} \sum_{w\in\mathcal{R}}{ \|\E z_w- \bar{z}\|^2} + \frac{2 \epsilon^2}{(W-2sB)^2}. \nonumber
	\end{align}
\end{lemma}

\begin{proof}
	We firstly prove \eqref{lm1-1-full}. As Lemma \ref{lemma:resampling} claims, when $B < \frac{W}{s}$, there exists a set $\mathcal{R}' \subseteq \mathcal{W}$ with at least $W-sB$ elements, such that for any $w' \in \mathcal{R}'$, \eqref{equation:resampling-expectation} and \eqref{equation:resampling-variance} hold true. When $B<\frac{W}{2s}$, $\frac{|\mathcal{R}'|}{W} > \frac{1}{2}$. With \cite[Lemma 2]{byrdsaga-arxiv}, it holds that
	\begin{align}
		\label{inequality:resample-1}
		&\E \|\mathop{\text{geomed}}(\{\tilde{z}_w, w\in \mathcal{W}\})-\bar z\|^2  \\
		% -------------------
		=& \E \|\mathop{\text{geomed}}(\{\tilde{z}_w-\bar z, w\in \mathcal{W}\})\|^2 \nonumber \\
		% -------------------
		\le& \frac{C_{s\alpha}^2}{|\mathcal{R'}|}\sum_{w'\in\mathcal{R}'} \E \|\tilde{z}_{w'}-\bar z\|^2.\nonumber
	\end{align}
	Applying Lemma \ref{lemma:resampling-inner-outer-variation} completes the proof of \eqref{lm1-1-full}.
	
	Next, we prove \eqref{lm1-1-noise}. With \cite[Lemma 3]{byrdsaga-arxiv}, it holds that
	\begin{align}
		&\E \| \tilde{z}_\epsilon^* - \bar{z}\|^2  \\
		% -------------------
		\leq& \frac{2C_{s\alpha}^2}{|\mathcal{R'}|} \sum_{w'\in\mathcal{R}'} \E \|\tilde{z}_{w'} - \bar{z}\| + \frac{2\epsilon^2}{(W - 2|\mathcal{R'}|)^2} \nonumber \\
		% -------------------
		\leq& \frac{2C_{s\alpha}^2}{|\mathcal{R'}|} \sum_{w'\in\mathcal{R}'} \E \|\tilde{z}_{w'} - \bar{z}\| + \frac{2\epsilon^2}{(W - 2sB)^2}. \nonumber
	\end{align}
	Applying Lemma \ref{lemma:resampling-inner-outer-variation} completes the proof of \eqref{lm1-1-noise}.	
\end{proof}

\section{Proof of Theorem \ref{Theorem:RS+Geomed+SAGA}}
To prove Theorem \ref{Theorem:RS+Geomed+SAGA}, we review the following supporting lemma for SAGA.
\begin{lemma}
	\label{lemma:SAGA-innerVariation}
	\cite[Lemmas 4 and 5]{byrdsaga-arxiv}
	Under Assumption \ref{assumption1}, if all regular workers $w \in \mathcal{W}$ update $\phi_{w,i_w^k}^k$ and $v_w^k$ according to \eqref{rule:phi} and \eqref{vvv}, then the corrected stochastic gradient $v_w^k$ satisfies
	\begin{align} \label{c1}
		\E\|v_w^k-f_w'(x^k)\|^2 \leq
		L^2 \frac{1}{J} \sum_{j=1}^{J} \|x^k - \phi_{w,j}^k\|^2, \quad \forall w \in \mathcal{W},
	\end{align}
	and
	\begin{align} \label{c1-1}
		\frac{1}{R} \sum_{w\in\mathcal{R}}\E\|v_w^k-f_w'(x^k)\|^2 \leq
		L^2 S^k,
	\end{align}
	where $S^k$ is defined as
	\begin{align}\label{c4}
		S^k := \frac{1}{R} \sum_{w \in \mathcal{R}} \frac{1}{J} \sum_{j=1}^{J} \|x^k - \phi_{w, j}^k\|^2.
	\end{align}
	Further, $S^k$ satisfies
	\begin{align}\label{d1}
		\E S^{k+1} \leq & 4J  \E \|x^{k+1} - x^k + \gamma f'(x^k) \|^2 \\
		& + 4J \gamma^2 L^2 \|x^k - x^*\|^2 + (1 - \frac{1}{J^2}) S^k. \nonumber
	\end{align}
\end{lemma}
%\begin{proof}
%    \begin{align}\label{c7}
%        & \E\|v_w^k-f_w'(x^k)\|^2 \nonumber \\
%        % -------------------
%        = & \E \|
%            f_{w,i_w^k}'(x^k) - f_{w,i_w^k}'(\phi_{w,i_w^k}^k)
%            + \frac{1}{J} \sum_{j=1}^{J} f_{w,i_w^k}'(\phi_{w,j}^k)
%            - f_w'(x^k)\|^2 \nonumber \\
%        % -------------------
%        = & \E\|f_{w,i_w^k}'(x^k) - f_{w,i_w^k}'(\phi_{w,i_w^k}^k)\|^2
%        - \E\|\frac{1}{J} \sum_{j=1}^{J} f_{w,i_w^k}'(\phi_{w,j}^k)
%            - f_w'(x^k)\|^2 \nonumber \\
%        % -------------------
%        \leq & \E \|f_{w,i_w^k}'(x^k) - f_{w,i_w^k}'(\phi_{w,i_w^k}^k)\|^2 \nonumber \\
%        % -------------------
%        \leq & L^2 \E \|x^k - \phi_{w,i_w^k}^k\|^2 \nonumber \\
%        % -------------------
%        = & L^2 \frac{1}{J} \sum_{j=1}^{J} \|x^k - \phi_{w,j}^k\|^2,
%    \end{align}
%    where the second equality is due to $\E \|a - \E a\|^2 = \E \|a\|^2 - \|\E a\|^2$ and the last inequality is due to Assumption \ref{assumption1}.
%\end{proof}

%Lemma \ref{lemma:SAGA-innerVariation} shows the variance of SAGA can be bounded by a time-varying term $S^k$. Now we will refer a lemma in \cite{byrdsaga} to depict the changing rule of $S^k$ without other proof.
%\begin{lemma}
%    \cite[Lemma 5]{byrdsaga} Following update rule \eqref{rule:phi} and under Assumption \ref{assumption1}, it holds that
%    %
%    \begin{align}\label{d1}
%        \E S^{k+1} \leq 4J \cdot \E \|x^{k+1} - x^k + \gamma f'(x^k) \|^2 + 4J \gamma^2 L^2 \|x^k - x^*\|^2 + (1 - \frac{1}{J^2}) S^k.
%    \end{align}
%    where $S^k$ is defined in \eqref{c4}.
%\end{lemma}

\begin{theorem}
	\label{Theorem:RS+Geomed+SAGA-full}(Full version of Theorem \ref{Theorem:RS+Geomed+SAGA})
	Under Assumptions \ref{assumption1} and \ref{assumption2}, if the number of Byzantine workers satisfies $B < \frac{W}{2s}$ and the step size satisfies
	$$
	\gamma \leq \frac{\mu}{2\sqrt{10} J^2 L^2 C_{s\alpha} },
	$$
	then the iterate $x^k$ generated by the proposed method in Algorithm 2 satisfies
	\begin{align}\label{th1-1-full}
		\E \|x^k - x^*\|^2 \leq (1 - \frac{\gamma\mu}{2})^k \Delta_1 + \Delta_2,
	\end{align}
	where
	\begin{align}\label{th1-2-full}
		\Delta_1 := \|x^0 - x^*\|^2 - \Delta_2,
	\end{align}
	\begin{align}\label{th1-3-full}
		\Delta_2 := \frac{5d C_{s\alpha}^2 \delta^2}{\mu^2},
	\end{align}
	while $\alpha :=\frac{B}{W}$, $C_{s\alpha}:= \frac{2-2s\alpha}{1-2s\alpha}$, $d: = \frac{W-1}{s W - 1}$, and $\E$ denote the expectation with respect to all random variables $i_w^k$ and the resampling processes. On the other hand, when the step size satisfies
	$$
	\gamma \leq \frac{\mu}{4\sqrt{5} J^2 L^2 C_{s\alpha} },
	$$
	the iterate $x^k$ generated by the proposed method in Algorithm 2 with $\epsilon$-approximate geometric median aggregation satisfies
	\begin{align}\label{th1-1-noise}
		\E \|x^k - x^*\|^2 \leq (1 - \frac{\gamma\mu}{2})^k \Delta_1 + 2\Delta_2 + \frac{10\epsilon^2}{\mu^2(W - 2 s B)^2}.
	\end{align}
\end{theorem}

\begin{proof}
	For simplicity, we only prove \eqref{th1-1-full}. According to the proof of Theorem 1 in \cite{byrdsaga-arxiv}, when
	\begin{align}\label{gamma-con-1}
		\gamma \leq \frac{\mu}{2L^2},
	\end{align}
	it holds that
	\begin{align}\label{e1}
		&\E \|x^{k+1} - x^*\|^2 \\
		\leq& (1-\gamma\mu) \|x^k - x^*\|^2 + \frac{2}{\gamma \mu} \E \|x^{k+1} - x^k + \gamma f'(x^k)\|^2. \nonumber
	\end{align}
	Then, we construct a \textit{Lyapunov function $T_1^k$} as
	\begin{align}\label{e2}
		T_1^k := \|x^k - x^*\|^2 + c_1S^k,
	\end{align}
	where $c_1$ is any positive constant. Since $S^k$ is non-negative according to the definition in \eqref{c4}, we know that $T_1^k$ is also non-negative. Substituting \eqref{d1} and \eqref{e1} into \eqref{e2} yields
	\begin{align}\label{e3}
		&\E T_1^{k+1} \\
		\leq & (1 - \gamma\mu + 4c_1J\gamma^2L^2) \|x^k - x^*\|^2 \nonumber\\
		+& (\frac{2}{\gamma\mu} + 4c_1J) \E \|x^{k+1} - x^k + \gamma f'(x^k)\|^2 \nonumber\\
		+& (1 - \frac{1}{J^2}) c_1 S^k. \nonumber
	\end{align}
	
	Note that the second term at the right-hand side of \eqref{e3} can be bounded with the help of Lemma \ref{lemma:resampling-geomed-full}, as
	\begin{align}\label{c5}
		&\E \|x^{k+1} - x^k + \gamma f'(x^k)\|^2  \\
		% -------------------
		= & \gamma^2 \E \|\mathrm{geomed}\left(\{\tilde{v}_w^k, w \in \mathcal{W}\} \right) - f'(x^k)\|^2 \nonumber \\
		% -------------------
		\leq & \gamma^2\left(d+\frac{1-d}{R}\right)\frac{C_{s\alpha}^2}{R} \sum_{w\in\mathcal{R}}{ \E\|v_w^k-f_w'(x^k)\|^2 } \nonumber \\
		&+ \gamma^2 \frac{dC_{s\alpha}^2}{R} \sum_{w\in\mathcal{R}}{ \|f_w'(x^k)- f'(x^k)\|^2} \nonumber \\
		% -------------------
		\leq & \gamma^2 \left(d+\frac{1-d}{R}\right)C_{s\alpha}^2L^2 S^k
		+ \gamma^2 dC_{s\alpha}^2\delta^2, \nonumber
	\end{align}
	where the last inequality comes from Lemma \ref{lemma:SAGA-innerVariation} and Assumption \ref{assumption2}. Therefore, we have
	\begin{align}
		\label{e5:0}
		&\E T_1^{k+1} \\
		\leq & (1 - \gamma\mu + 4c_1J\gamma^2L^2) \|x^k - x^*\|^2 \nonumber \\
		+ &\bigg(1 - \frac{1}{J^2}\bigg)c_1S^k  \nonumber\\
		+ &\bigg(\frac{2}{\gamma\mu} + 4c_1J\bigg) \bigg( d + \frac{1-d}{R}\bigg) C_{s\alpha}^2\gamma^2L^2S^k. \nonumber \\
		+ &\gamma^2 \bigg(\frac{2}{\gamma\mu} + 4c_1J\bigg) d C_{s\alpha}^2 \delta^2 \nonumber
	\end{align}
	If we constrain the step size $\gamma$ as
	\begin{align}\label{e6}
		4c_1J\gamma^2L^2 \leq \frac{\gamma\mu}{2},
	\end{align}
	then the coefficients in \eqref{e5:0} satisfy
	\begin{align}\label{e7}
		1 - \gamma\mu +  4c_1J\gamma^2L^2 & \leq 1 - \frac{\gamma\mu}{2}, \\
		\frac{2}{\gamma\mu} + 4c_1J
		& \le \frac{2}{\gamma\mu} + \frac{\mu}{2\gamma L^2}
		\le \frac{5}{2\gamma\mu}.
	\end{align}
	Therefore, we can bound $\E T^{k+1}$ by
	\begin{align}\label{e5}
		\E T_1^{k+1} &\leq (1 - \frac{\gamma\mu}{2}) \|x^k - x^*\|^2 + \frac{5\gamma}{2\mu} d C_{s\alpha}^2 \delta^2 \\
		&+ \bigg(\bigg(1 - \frac{1}{J^2}\bigg)c_1 + \frac{5\gamma}{2\mu}\bigg(d + \frac{1-d}{R}\bigg)C_{s\alpha}^2 L^2\bigg)S^k. \nonumber
	\end{align}
	Similarly, if $\gamma$ and $c_1$ are chosen such that
	\begin{align}\label{e8}
		\frac{\gamma\mu}{2} < \frac{1}{2J^2},
	\end{align}
	and
	\begin{align}\label{e9}
		c_1
		= & \frac{5 J^2  \gamma L^2 (d + (1-d) / R) C_{s\alpha}^2}{\mu} \\
		\geq& \frac{5 \gamma L^2 (d + (1-d) / R)C_{s\alpha}^2/2}{\mu (1/J^2 - \gamma\mu/2)}, \nonumber
	\end{align}
	then the coefficient in \eqref{e5} satisfies
	\begin{align}\label{e10}
		&\bigg(1 - \frac{1}{J^2}\bigg)c_1 + \frac{5\gamma}{2\mu}\bigg(d +  \frac{1-d}{R}\bigg) C_{s\alpha}^2 L^2 \leq (1 - \frac{\gamma\mu}{2}) c_1.
	\end{align}
	Hence, \eqref{e5} becomes
	\begin{align}\label{e11}
		&\E T_1^{k+1} \\
		\leq & (1 - \frac{\gamma\mu}{2}) \|x^k - x^*\|^2 + (1 - \frac{\gamma\mu}{2}) c_1 S^k
		+ \frac{5\gamma}{2\mu}
		d C_{s\alpha}^2 \delta^2\nonumber \\
		% -----------------------
		= & (1 - \frac{\gamma\mu}{2}) T_1^k
		+ \frac{5\gamma}{2\mu}
		d C_{s\alpha}^2 \delta^2. \nonumber
	\end{align}
	%
	%        For simplicity, let also
	%        %
	%        \begin{align}\label{e12}
	%            \tilde{\Delta}_2 := \frac{5\gamma}{\mu}
	%            \bigg( d C_{s\alpha}^2 \delta^2 +  \frac{\epsilon^2}{(W - 2 s B)^2} \bigg).
	%        \end{align}
	%        %
	
	Using telescopic cancellation on \eqref{e11} from time 1 to time $k$, we have
	\begin{align}\label{e13}
		\E T_1^k  \leq (1 - \frac{\gamma \mu}{2})^k \bigg[T_1^0 - \Delta_2 \bigg] + \Delta_2.
	\end{align}
	Here and thereafter, the expectation is taken over $i_w^k$ for all regular workers $w \in \mathcal{R}$. The definition of the \textit{Lyapunov function} in \eqref{e2} implies that
	\begin{align}\label{e14}
		\E \|x^k - x^*\|^2 \leq \E T_1^k \leq (1 - \frac{\gamma \mu}{2})^k \Delta_1 + \Delta_2.
	\end{align}
	%        %
	%        where the constant $\Delta_1$ and $\Delta_2$ are defined as
	%        %
	%        \begin{align}\label{e15}
	%            \Delta_1 := \|x^0 - x^*\| - \Delta_2,
	%        \end{align}
	%        %
	%        %
	%        \begin{align}
	%            \Delta_2 := \frac{2}{\gamma\mu} \tilde{\Delta}_2
	%            = \frac{10}{\mu^2}
	%            \bigg( d C_{s\alpha} \delta^2 +  \frac{\epsilon^2}{(W - 2 s B)^2} \bigg).
	%        \end{align}
	%        %
	
	In our derivation so far, the constraint on the step size $\gamma$ (c.f. \eqref{gamma-con-1}, \eqref{e6} and \eqref{e8}) is
	\begin{align}
		\gamma \leq \min &\bigg\{ \frac{\mu}{2L^2}, \frac{1}{J^2 \mu},  \\
		&\frac{\mu}{2\sqrt{10}J^{3/2}L^2C_{s\alpha} (d + (1-d) / R)^{1/2}}
		\bigg\}. \nonumber
	\end{align}
	Therefore, we simply choose
	\begin{align}
		\gamma \leq \frac{\mu}{2\sqrt{10} J^2 L^2 C_{s\alpha} },
	\end{align}
	which completes the proof.
\end{proof}

\section{Convergence Property of RS-Byrd-SGD}

Now we show that the learning error of RS-Byrd-SGD, whose update is given by
\begin{align}\label{eq:rs-byrd-sgd}
	x^{k+1} = x^k - \gamma  \cdot \mathrm{geomed}\left(\{\tilde{g}_w^k, w \in \mathcal{W}\}\right),
\end{align}
is in the order of ${\cal O}((d + \frac{1-d}{R}) C_{s\alpha}^2 \sigma^2 + d C_{s\alpha}^2 \delta^2)$.
\begin{theorem}
	\label{Theorem:RS-Byrd-SGD} Under Assumptions \ref{assumption1}, \ref{assumption2} and \ref{assumption3}, if the number of Byzantine workers satisfies $B < \frac{W}{2s}$ and the step size satisfies
	$$
	\gamma \leq \frac{\mu}{2 L^2},
	$$
	then the iterate $x^k$ generated by RS-Byrd-SGD satisfies
	\begin{align}\label{th-RS-Byrd-SGD-1}
		\E \|x^k - x^*\|^2 \leq (1 - \gamma\mu)^k \Delta_1 + \Delta_2,
	\end{align}
	where
	\begin{align}\label{th-RS-Byrd-SGD-2}
		\Delta_1 := \|x^0 - x^*\|^2 - \Delta_2,
	\end{align}
	\begin{align}\label{th-RS-Byrd-SGD-3}
		\Delta_2 := \frac{2}{\mu^2} \bigg(\bigg(d + \frac{1-d}{R}\bigg) C_{s\alpha}^2\sigma^2 + d C_{s\alpha}^2\delta^2 \bigg).
	\end{align}
	while $\alpha :=\frac{B}{W}$, $C_{s\alpha}:= \frac{2-2s\alpha}{1-2s\alpha}$, $d: = \frac{W-1}{s W - 1}$, and $\E$ denote the expectation with respect to all random variables $i_w^k$ and the resampling processes. On the other hand, the iterate $x^k$ generated by RS-Byrd-SGD with $\epsilon$-approximate geometric median aggregation satisfies
	\begin{align}\label{th-RS-Byrd-SGD-1-noise}
		\E \|x^k - x^*\|^2 \leq (1 - \gamma\mu)^k \Delta_1 + 2 \Delta_2 + \frac{4\epsilon^2}{\mu^2(W - 2sB)^2}.
	\end{align}
\end{theorem}

\begin{proof}
	For simplicity, we only prove \eqref{th-RS-Byrd-SGD-1}. Note that inequality \eqref{e1} is still true for RS-Byrd-SGD with $\gamma < \frac{\mu}{2L^2}$. The only difference is that we need to bound the term $\E \|x^{k+1} - x^k +\gamma f'(x^k)\|$ with Lemma \ref{lemma:resampling-geomed-full}, as
	\begin{align}\label{proof-RS-Byrd-SGD-1}
		&\E \|x^{k+1} - x^k +\gamma f'(x^k)\| \\
		=& \gamma^2 \E \|\mathrm{geomed}\left(\{\tilde{g}_w^k, w \in \mathcal{W}\} \right) -f'(x^k)\|^2 \nonumber \\
		\leq & \gamma^2  \bigg( \bigg(d+\frac{1-d}{R}\bigg)\frac{C_{s\alpha}^2}{R}\sum_{w\in\mathcal{R}}
		\E \|f'_{w, i_w^k} (x^k) - f'_w(x^k)\|^2 \nonumber\\
		&+ \frac{dC_{s\alpha}^2}{R}\sum_{w\in\mathcal{R}} \|f'_w(x^k)- f'(x^k)\|^2\bigg)   \nonumber \\
		\leq & \gamma^2  \bigg( \bigg(d + \frac{1-d}{R}\bigg)C_{s\alpha}^2\sigma^2 + dC_{s\alpha}^2\delta^2 \bigg). \nonumber
	\end{align}
	Therefore, we have
	\begin{align}\label{f6}
		&\E \|x^{k+1} - x^*\|^2  \\
		\leq & (1-\gamma\mu) \|x^k - x^*\|^2 + \frac{2}{\gamma \mu} \E \|x^{k+1} - x^k + \gamma f'(x^k)\|^2 \nonumber \\
		% ---------------------------
		\leq & (1-\gamma\mu) \|x^k - x^*\|^2  \nonumber\\
		&+ \frac{2\gamma}{ \mu} \bigg(\bigg(d + \frac{1-d}{R}\bigg)C_{s\alpha}^2\sigma^2 + dC_{s\alpha}^2\delta^2\bigg). \nonumber
	\end{align}
	Here and thereafter, the expectation is taken over $i_w^k$ for all regular workers $w \in \mathcal{R}$. Using telescopic cancellation on \eqref{f6} completes the proof.
\end{proof}

% use section* for acknowledgment
%\section*{Acknowledgment}

% Can use something like this to put references on a page
% by themselves when using endfloat and the captionsoff option.
\ifCLASSOPTIONcaptionsoff
  \newpage
\fi

% trigger a \newpage just before the given reference
% number - used to balance the columns on the last page
% adjust value as needed - may need to be readjusted if
% the document is modified later
%\IEEEtriggeratref{8}
% The "triggered" command can be changed if desired:
%\IEEEtriggercmd{\enlargethispage{-5in}}

% references section

% can use a bibliography generated by BibTeX as a .bbl file
% BibTeX documentation can be easily obtained at:
% http://mirror.ctan.org/biblio/bibtex/contrib/doc/
% The IEEEtran BibTeX style support page is at:
% http://www.michaelshell.org/tex/ieeetran/bibtex/
\bibliographystyle{IEEEtran}
% argument is your BibTeX string definitions and bibliography database(s)
\bibliography{IEEEabrv,refs}
\end{document}